\documentclass[journal,12pt,draftclsnofoot,onecolumn]{IEEEtran}
\usepackage{mathrsfs}
\usepackage{amsmath}
\usepackage{amsthm}
\usepackage{tabularx}
\usepackage{booktabs}
\usepackage{amssymb}
\usepackage{graphicx}
\usepackage{subfigure}
\usepackage{times}
\usepackage{latexsym}
\usepackage{subfigure}
\usepackage{cite}
\usepackage{balance}
\usepackage{multirow}
\usepackage{algorithmic}
\usepackage{color}
\usepackage{float}
\usepackage{cases}
\usepackage{algorithm}
\usepackage{yhmath}
\usepackage{bm}
\usepackage{footnote}
\usepackage[justification=centering]{caption}
\newtheorem{lemma}{Lemma}
\newtheorem{theorem}{Theorem}

\theoremstyle{remark}
\newtheorem{remark}{Remark}

\begin{document}
\title{1-Bit Compressive Sensing for Efficient Federated Learning Over the Air}
\author{ Xin Fan$^{1}$, Yue Wang$^2$,~\IEEEmembership{Member,~IEEE}, Yan Huo$^{1}$,~\IEEEmembership{Senior Member,~IEEE}, and Zhi Tian$^2,~\IEEEmembership{Fellow,~IEEE}$\\
$^{1}$School of Electronics and Information Engineering, Beijing Jiaotong University, Beijing, China\\
$^2$Department of Electrical \& Computer Engineering, George Mason University, Fairfax, VA, USA
\\E-mails: \{yhuo,fanxin\}@bjtu.edu.cn, \{ywang56,ztian1\}@gmu.edu}
\maketitle

\begin{abstract}
For distributed learning among collaborative users, this paper develops and analyzes a communication-efficient scheme for federated learning (FL) over the air, which incorporates 1-bit compressive sensing (CS) into analog aggregation transmissions. To facilitate design parameter optimization, we theoretically analyze the efficacy of the proposed scheme by deriving a closed-form expression for the expected convergence rate of the FL over the air. Our theoretical results reveal the tradeoff between convergence performance and communication efficiency as a result of the aggregation errors caused by sparsification, dimension reduction, quantization, signal reconstruction and noise.
Then, we formulate 1-bit CS based FL over the air as a joint optimization problem to mitigate the impact of these aggregation errors through joint optimal design of worker scheduling and power scaling policy. An enumeration-based method is proposed to solve this non-convex problem, which is optimal but becomes computationally infeasible as the number of devices increases. For scalable computing, we resort to the alternating direction method of multipliers (ADMM) technique to develop an efficient implementation that is suitable for large-scale networks. Simulation results show that our proposed 1-bit CS based FL over the air achieves comparable performance to the ideal case where conventional FL without compression and quantification is applied over error-free aggregation, at much reduced communication overhead and transmission latency.  
\end{abstract}
\begin{IEEEkeywords}
Federated learning, analog aggregation, 1-bit compressive sensing, convergence analysis, joint optimization.
\end{IEEEkeywords}

\section{Introduction}
Centralized machine learning (ML) that collects distribute data from edge devices (local workers) to a parameter server (PS) for data analysis and inference is becoming increasingly costly for communications. Although the adoption of approximate data aggregation instead of exact data aggregation proposed by \cite{he2015approximate,li2017approximate} can effectively reduce communication costs, the privacy issues exposed by data collection cannot be ignored.
As an alternative, federated learning (FL) is a promising paradigm that enables many local workers to collaboratively train a common learning model under the coordination of a PS in wireless networks \cite{mcmahan2017communication,yang2019federated,chen2020joint}. In each round of iteration-based FL, starting from a common learning model received from the PS, local devices (workers) proceed to train the model with their own local data by updating their local model parameters, and then transmit their local updates to the PS. Next, all the collected local updates are averaged at the PS and then sent back to local workers for the next round local updates. Without exchanging raw datasets during the iterations between the PS and local workers, FL offers distinct advantages on protecting user privacy and leveraging distributed on-device computation compared to traditional learning at a centralized data center.


In FL, updates shared between local workers and the PS can be extremely large, e.g., the VGGNet architecture has approximately 138 million parameters. As a result, the pre-processing of updates has been considered in the literature to reduce the communication load per worker,
 such as \emph{sparsification}, \emph{quantization} and \emph{communication censoring} schemes. Given the compressible nature of gradients where only a small percentage of entries have large values while the rest remain at relatively small values \cite{lin2018deep}, \emph{sparsification} schemes only keep the large values of local updates to reduce the communication load\cite{aji2017sparse, lin2017deep}. \emph{Quantization} is to compress the continuous-valued update information to a few finite bits so that it can be effectively communicated over a digital channel\cite{liu2019decentralized,seide20141, alistarh2017qsgd}. \emph{Communication censoring} is to evaluate the importance of each update in order to avoid less informative transmissions\cite{liu2019communication,8755802,chen2018lag,8646657,xu2020coke}. All these useful strategies are investigated predominantly for FL with digital communications. 


However, the communication overhead and transmission latency of FL over digital communication channels are still proportional to the number of active workers, and thus cannot be applicable in large-scale environments. To overcome this problem, an analog aggregation model is recently proposed for FL\cite{zhu2019broadband,cao2019optimal,yang2020federated,zhu2020one,amiri2020machine, amiri2020federated1,amiri2019collaborative,sun2019energy} by allowing multiple workers to simultaneously transmit their updates over the same time-frequency resources and then applying an average-enabled computation-over-the-air principle\cite{nazer2007computation}. It benefits from the fact that FL only relies on the averaged value of distributed local updates rather than their individual values. Exploiting the waveform superposition property of a wireless multiple access channel (MAC), analog aggregation automatically enables to directly obtain the averaged updates required by FL, which prompts the prosperity of analog aggregation based FL. 
In \cite{zhu2019broadband}, a broadband analog aggregation scheme was designed for FL, in which a set of tradeoffs between communication and learning are derived for broadband power control and device scheduling, where the learning metric is set as the fraction of scheduled devices.  In \cite{cao2019optimal}, a jointly optimization was proposed to minimize the mean square error (MSE) of the aggregated signal. Similarly, a joint design of device scheduling and beamforming was presented in \cite{yang2020federated} for FL over
the air in multiple antenna systems, which aims to maximize the number of selected workers under the given MSE requirements.  Based on one-bit gradient quantization, a digital version of broadband over-the-air aggregation was proposed, and the effects of wireless channel hostilities on the convergence rate was analyzed in \cite{zhu2020one}. 
In \cite{amiri2020machine, amiri2020federated1}, the gradient sparsification, and a random linear projection for dimensionality reduction of large-size gradient in narrow-band channels was considered to reduce the communication requirements. The power allocation scheme in \cite{amiri2020machine, amiri2020federated1} scales the power of the vectors containing the gradient information of different
devices to satisfy the average power constraint. 


Despite the prior work, some fundamental questions remain  unanswered, which however prevent from achieving communication-efficient and high-performance FL with analog aggregation.
Firstly, the quantitative relationship between FL and analog aggregation communication is not clear. Simple maximization of the number of participated workers is learning-agnostic and hence not necessarily optimal, which decouples the optimization of computation and communication, e.g., the works in \cite{sun2019energy,yang2020federated}. Secondly, to facilitate power control, most existing works are developed based on a strong assumption that the signals to be transmitted from local workers, i.e., local gradients, can be normalized to have zero mean and unit variance\cite{cao2019optimal,zhu2019broadband,yang2020federated,zhu2020one}. However, gradient statistics in FL vary over both training iterations and feature dimensions, and are unknown a priori\cite{zhang2020gradient}. Thus, it is infeasible to design an optimal power control without prior knowledge of the local gradients at the PS, especially for the non-coding linear analog modulation in analog aggregation based FL. Thirdly, sparsification is introduced for communication efficiency in analog aggregation based FL \cite{amiri2020machine, amiri2020federated1} as a means of lossy compression of local gradients, which may introduce aggregation errors, but the impact of these aggregation errors on FL is not yet clear let alone how to alleviate their side effects. 

To solve the aforementioned issues, in this paper, we introduce 1-bit compressive sensing (CS) for efficient FL over the air, by developing an optimized practical worker selection and power control policy. To the best of our knowledge, this is the first work to introduce 1-bit CS \cite{boufounos20081,jacques2013robust,dai2016noisy} into FL over the air for high communication efficiency, where both the dimension of local gradients and the number of quantization bits can be reduced significantly.
Further, thanks to the 1-bit quantization, our power control becomes feasible since it hinges on the quantized values of known magnitude, without relying on any prior knowledge or assumptions on gradient statistics or specific distribution.
More importantly, our work provides an essential interpretation on the relationship between FL and analog aggregation with 1-bit CS techniques to enable joint optimization of computation and communications. Our main contributions are outlined below:

\begin{itemize}
\item We propose an \textbf{o}ne-\textbf{b}it \textbf{CS} \textbf{a}nalog \textbf{a}ggregation (OBCSAA) for efficient FL. In our OBCSAA, we elaborately design a set of preprocessing, analog aggregation transmission, signal reconstruction solutions to achieve communication-efficient FL.
\item We derive a closed-form expression for the expected convergence rate of our OBCSAA. This closed-form expression measures the performace tradeoff as a result of the aggregation errors caused by sparsification, dimension reduction, quantization, signal reconstruction and additive white gaussian noise, which provides a fresh perspective to design analog wireless systems.

    \item Guided by the theoretical results, we formulate a joint optimization problem of computation and communication to optimize the worker selection and power control. Given the practical limitation on allowable peak transmit power and available bandwidth, this optimization problem aims to mitigate the aggregation errors. To solve this non-convex optimization problem, we propose two solutions: the enumeration-based method and the alternating direction method of multipliers (ADMM) approach for the scenarios of small networks and large networks, respectively. 
\end{itemize}

We evaluate the proposed OBCSAA in solving image classification problems on the MNIST dataset. Simulation results show that our proposed OBCSAA achieves comparable performance to the ideal case where FL is implemented by perfect aggregation over error-free wireless channels, with much enhanced communication efficiency.

It is worth noting that, unlike its digital communication counterpart, the optimization design for FL over the air faces a much reduced degree of freedom due to analog aggregation, and is not yet well explored in the literature \cite{zhu2019broadband,cao2019optimal,yang2020federated,zhu2020one,amiri2020machine, amiri2020federated1}.  Compared to \cite{zhu2019broadband,cao2019optimal,yang2020federated} that consider the fraction of scheduled devices as the learning metric which separates communication and computation, our learning metric is learning convergence with respect to CS and communication factors, which hence provides the exact relationship between communication and computation. Different from \cite{zhu2019broadband,cao2019optimal,yang2020federated,zhu2020one} developed on the assumption that the local updates have to follow independent and identically distributed (IID) with zero mean and unit variance, our work adopts 1-bit CS, which enables to achieve power control for individual workers even without any gradient statistical information required by \cite{zhu2019broadband,cao2019optimal,yang2020federated,zhu2020one}. Compared to \cite{amiri2020machine, amiri2020federated1}, our work not only applies the 1-bit quantization after dimensionality reduction, but also provides a convergence analysis on 1-bit CS based FL over the air, which leads to a joint optimization of computation and communications.In short, our work is a holistic integration of gradient sparsification, dimensionality reduction, and quantization for efficient FL over the air.



The rest of this paper is organized as follows. The system model of 1-bit compressive sensing for FL over the air is presented in Section \ref{sec:Model}. The closed-form expression of the expected convergence rate is derived in Section \ref{sec:Convergence Analysis} to quantify the impact of the aggregation errors on FL. A joint optimization problem of communication and FL to optimize worker selection and power control are studied in Section \ref{sec:Joint optimization}. Numerical results are presented in Section \ref{Sec:Numerical Results}, and conclusions are drawn in Section \ref{Sec:Conclusion}.

\section{System Model}\label{sec:Model}
We consider a wireless FL system consisting of a single PS and $U$ local workers. Exploiting wireless analog aggregation transmissions with 1-bit CS, the PS and all local workers collaboratively train a shared learning model.

\subsection{FL Model}
Suppose that the union of all training datasets is denoted as $\mathcal{D} = \bigcup_i \mathcal{D}_i $, where $\mathcal{D}_i=\{\mathbf{x}_{i,k},\mathbf{y}_{i,k}\}_{k=1}^{K_i}$ is the local dataset and $K_i = |\mathcal{D}_i|$ is the number of data samples at the $i$-th worker, $i=1, \ldots, U$. In $\mathcal{D}_i$, the $k$-th data sample and its label are denoted as $\mathbf{x}_{i,k}$ and $\mathbf{y}_{i,k}$, $k=1,2,...,K_i$, respectively. The objective of the training procedure is to minimize the global loss function $F(\mathbf{w}; \mathcal{D})$ of the global shared learning model parameterized by $\mathbf{w} = [w^1, \ldots, w^D] \in \mathcal{R}^D$ of the dimension $D$, i.e.,
\begin{align}\label{eq:lossfopt}
 \textbf{P1:} \quad \mathbf{w}^*=\arg \min_{\mathbf{w}\in \mathcal{R}^D}& \quad F(\mathbf{w}; \mathcal{D}),
\end{align}
where $F(\mathbf{w}; \mathcal{D})=\frac{1}{K}\sum_{i=1}^U\sum_{k=1}^{K_i} f(\mathbf{w};\mathbf{x}_{i,k},\mathbf{y}_{i,k})$ is the summation of $K=\sum_{i=1}^U K_i$ sample-wise loss functions defined by the learning model.

To avoid directly uploading the raw local datasets to the PS for central training, the learning procedure in (\textbf{P1}) is conducted in a distributed manner by an
iterative gradient-averaging algorithm \cite{konevcny2016federated,mcmahan2017communication}.  Specifically, at each iteration $t$, the gradient descent (GD)\footnote{In this work, we take the basic gradient descent as an example, which can be extended to the stochastic gradient descent (SGD) by using a mini-batch at each worker for training. Note that SGD works more computation-efficient at the cost of more iterations and hence more transmissions compared to GD.} is applied at local workers in parallel to minimize the local loss functions
\begin{equation}\label{eq:locallossf}
 \text{(Local loss function)}\quad F_i(\mathbf{w}_i;\mathcal{D}_i)=\frac{1}{K_i}\sum_{k=1}^{K_i} f(\mathbf{w}_i;\mathbf{x}_{i,k},\mathbf{y}_{i,k}), \quad i=1,..., U,
\end{equation}
where $\mathbf{w}_i = [w_i^1, \ldots, w_i^D] \in \mathcal{R}^D$ is the local model parameter.
Each local worker updates its local gradient from the received global learning model given its own local dataset:
\begin{align}\label{eq:localgradient0}
  \text{(Local gradient computing)} \quad \mathbf{g}_i=\frac{1}{K_i}\sum_{k=1}^{K_i}\nabla f(\mathbf{w}_i;\mathbf{x}_{i,k},\mathbf{y}_{i,k}),\quad i=1,..., U,
\end{align}
where $\nabla f(\mathbf{w}_i;\mathbf{x}_{i,k},\mathbf{y}_{i,k})$ is the gradient of $f(\mathbf{w}_i;\mathbf{x}_{i,k},\mathbf{y}_{i,k})$ with respect to $\mathbf{w}_i$.

Then the local gradients are sent to the PS, which are aggregated as the global gradient:
\begin{align}\label{eq:globalgradient0}
\text{(Global gradient computing)}\quad \mathbf{g}=\frac{1}{K}\sum_{i=1}^{U} K_i\mathbf{g}_i,
\end{align}
and the global gradient $\mathbf{g}$ is sent back to the local workers, which is then used to update the shared model as
\begin{align}\label{eq:sharedupdate0}
  \text{(Shared model updating)}\quad \mathbf{w}&=\mathbf{w}-\alpha \mathbf{g},
\end{align}
where $\alpha$ is the learning rate.

%
%

The FL implements \eqref{eq:localgradient0}, \eqref{eq:globalgradient0} and \eqref{eq:sharedupdate0} iteratively, until it converges or the maximum number of iterations is reached. 


\subsection{Analog Aggregation Transmission Model}
  In the scenarios of FL applied over large-scale networks and for training a high-dimensional model parameters, the transmissions between the PS and local workers consume a lot of communication resources and cause training latency. Meanwhile, due to the transmit power and bandwidth limitations posed by practical wireless communications, the digital communication approach of transmitting and reconstructing all the gradient entries one-by-one in an individual manner is an overkill. Thus, in order to reduce the transmission overhead and speed up communication time, we propose to apply 1-bit compressive sensing \cite{boufounos20081,jacques2013robust,dai2016noisy} in FL over the air,
  which is motivated by two facts.
One is that the gradients involved in large-size learning problems usually turn out to be compressible with only a small number of entries having significant values \cite{lin2018deep}.
The other is that FL is usually running in an averaged-based distributed learning mechanism.
In our work, through gradient sparsification, the compression nature of CS allows to reduce the dimension of the transmitted gradient vectors. Meanwhile, analog aggregation enables all local workers to simultaneously use the same time-frequency resources to transmit their updates to the PS. Further, the 1-bit quantization not only minimizes the quantization overhead, but also circumvents the unrealistic requirement on known distribution of local gradients. The procedure of the proposed 1-bit CS method for FL is elaborated next.


\subsubsection{Sparsification}
Before transmission at the $t$-th iteration, all local workers set all but the $\kappa$ elements of their local $\mathbf{g}_{i,t}$'s to 0, resulting $\kappa$-level sparsification denoted by
\begin{align}\label{eq:sparse}
  \tilde{\mathbf{g}}_{i,t}=\texttt{sparse}_\kappa(\mathbf{g}_{i,t}) ,
\end{align}
where $\texttt{sparse}_\kappa(\cdot)$ is a sparsification operation of a vector such that $\tilde{\mathbf{g}}_{i,t}$ is of length $D$ and sparsity order $\kappa$. In our paper, we perform a top-$\kappa$ sparsification strategy, i.e., elements with the largest $\kappa$ magnitudes are retained while other elements are set to 0.
\subsubsection{Dimension Reduction}
To transmit the non-zero entries of their sparsified local gradient vectors, the workers need to transmit the indices and values of the non-zero entries to the PS separately, which results in additional data transmissions.
 To avoid this overhead, all workers employ the same measurement matrix $\bm{\Phi}\in \mathbb{R}^{S\times D}$ ($S \ll D$) that is a random Gaussian matrix. Note that the specific $\kappa$-nonzero indices of sparse gradients after the top-$\kappa$ sparsification are usually different worker-by-worker\footnote{When distributed workers have i.i.d. data, their $\kappa$-nonzero indices turn to appear with large overlapping}, which results in an increased sparsity-level $\bar{\kappa}$ $(>\kappa)$ for the superposition gradient signal. For reliable reconstruction of the compressed gradients, it is desired that the restricted isometry property (RIP) condition be met, that is, $\kappa U\leq S \ll D$ and each entry of $\bm{\Phi}$ i.i.d. follows $\mathcal{N}(0, \sigma_{sp}^2)$, where $\kappa U$ is the upper bound of sparisty in the combined sparse gradient, i.e., $\kappa U > \bar{\kappa}$. In addition, $\bm{\Phi}$ is shared between the workers and the PS before transmissions.



\subsubsection{Quantization}
  Next, 1-bit quantization is applied to $\bm{\Phi}\tilde{\mathbf{g}}_{i,t}$'s, so that the resulting compressed local gradient $\mathcal{C}(\mathbf{g}_{i,t})$ at each worker is given by
\begin{align}\label{eq:finalcompress}
\mathcal{C}(\mathbf{g}_{i,t})&=\texttt{sign}(\bm{\Phi} \texttt{sparse}_\kappa(\mathbf{g}_{i,t}))\nonumber \\
&=\texttt{sign}(\bm{\Phi} \tilde{\mathbf{g}}_{i,t}),\quad i=1,..., U,
\end{align}
where $\mathcal{C}(\cdot)$ represents the overall effective operation including top-$\kappa$ sparsification, CS compression, and 1-bit quantization. 
\subsubsection{Analog Aggregation Transmission}
 After the above collecting the compressive measurements in \eqref{eq:finalcompress}, all the workers transmit their local $\mathcal{C}(\mathbf{g}_{i,t})$'s in an analog fashion, which are aggregated over the air at the PS to implement the global gradient computing step in \eqref{eq:globalgradient0}. 
Specifically, each local $\mathcal{C}(\mathbf{g}_{i,t})$ is multiplied with a pre-processing power control factor, denoted as $p_{i,t}$. Then, the received signal vector at the PS is given by
\begin{align}\label{eq:yt}
  \mathbf{y}_{t}&=\sum_{i=1}^U h_{i,t} p_{i,t}\mathcal{C}(\mathbf{g}_{i,t})+\mathbf{z}_{t},
\end{align}
where $\mathbf{z}_{t} \sim \mathcal{N}(0,\sigma^2\mathbf{I})$ is additive white Gaussian noise (AWGN) vector, and $h_{i,t}$ denotes the channel coefficient between the $i$-th local worker and the PS at the $t$-th iteration\footnote{In this paper, we consider block fading channels, where the channel state information (CSI) remains unchanged within each iteration in FL, but may independently vary from one iteration to another. We assume that the CSI is perfectly known at both the PS and local workers.}. 

Let $\beta_{i,t}$ denote the scheduling indicator, i.e., $\beta_{i,t}=1$ indicates that the $i$-th worker at the $t$-th iteration is scheduled to the FL algorithm, and $\beta_{i,t}=0$, otherwise. To implement the averaging gradient step in \eqref{eq:globalgradient0}, the signal vector of interest at the PS at the $t$-th iteration is given by
\begin{align}\label{eq:ydesired}
  \mathbf{y}^{desired}_t&=\frac{\sum_{i=1}^U K_i\beta_{i,t} \mathcal{C}(\mathbf{g}_{i,t})}{\sum_{i=1}^U K_i\beta_{i,t}}.
\end{align}

To obtain the signal vector of interest, we design the pre-processing power control factor $p_{i,t}$ as
\begin{equation}\label{power_control}
 p_{i,t}=\frac{\beta_{i,t}K_ib_{t}}{h_{i,t}},
\end{equation}
where $b_{t}$ is a power scaling factor. Through this power scaling, the transmit power at the $i$-th local worker satisfies the power limitation $P_i^{\text{Max}}$ as
\begin{equation}\label{power_limitation}
 |p_{i,t}c^s_{i,t}|^2=\left(\frac{\beta_{i,t}K_ib_{t}}{h_{i,t}}c^s_{i,t}\right)^2=\frac{\beta^2_{i,t}K^2_ib^2_{t}}{h^2_{i,t}}\leq P_i^{\text{Max}},
\end{equation}
where $c^s_{i,t} = \pm 1$ due to 1-bit quantization as the $s$-th element of $\mathcal{C}(\mathbf{g}_{i,t})=[c^1_{i,t},..., c^s_{i,t},...,c^S_{i,t}]^T$. As we can see from \eqref{power_limitation}, the power limitation is independent of the specific local gradient, which enables the optimization of power control to get rid of the prior knowledge on gradient or gradient statistics.    

After applying the pre-processing power control $p_{i,t}$ and substituting \eqref{power_control} into \eqref{eq:yt}, the received signal vector of \eqref{eq:yt} can be rewritten as

By such design of $p_{i,t}$, the received signal vector at the PS is rewritten as
\begin{align}\label{eq:ytre}
  \mathbf{y}_{t}&=\sum_{i=1}^U  K_i b_{t}\beta_{i,t}\mathcal{C}(\mathbf{g}_{i,t}) +\mathbf{z}_{t}.
\end{align}

Upon receiving $\mathbf{y}_{t}$, the PS estimates the signal vector of interest via a post-processing operation as
\begin{align}\label{eq:gt}
  \mathbf{\hat{y}}^{desired}_t=  &(\sum_{i=1}^U K_i\beta_{i,t}b_{t})^{-1}\mathbf{y}_{t}
  =(\sum_{i=1}^U K_i\beta_{i,t})^{-1}\sum_{i=1}^U K_i\beta_{i,t} \mathcal{C}(\mathbf{g}_{i,t})
  +(\sum_{i=1}^U K_i\beta_{i,t}b_{t})^{-1}\mathbf{z}_{t},
\end{align}
where $(\sum_{i=1}^U K_i\beta_{i,t}b_{t})^{-1}$ is the post-processing factor.

\subsubsection{Reconstruction}
After obtaining $\mathbf{\hat{y}}^{desired}_t$ from \eqref{eq:gt}, the PS needs to further use a 1-bit CS reconstruction algorithm $\mathcal{C}^{-1}(\cdot)$ (e.g., binary iterative hard thresholding (BIHT) algorithm \cite{jacques2013robust}, fixed point continuation algorithms \cite{hale2007fixed}, basis pursuit algorithms \cite{moshtaghpour2015consistent} and other greedy matching pursuit algorithms \cite{donoho2006compressed}) to estimate the global gradient $\hat{\mathbf{g}}_t=\mathcal{C}^{-1}(\mathbf{\hat{y}}^{desired}_t)$.
 Then the PS broadcasts the estimated $\hat{\mathbf{g}}_t$ to all the local workers
for updating the shared model parameter as follows
\begin{align}\label{eq:sharedupdate1}
 \mathbf{w}_{t+1}&=\mathbf{w}_t-\alpha \hat{\mathbf{g}}_t.
\end{align}

Compared \eqref{eq:sharedupdate1} and \eqref{eq:sharedupdate0}, aggregation errors may be introduced in 1-bit CS based FL over the air, due to analog aggregation transmissions, top-$\kappa$ sparsification, CS compression, and 1-bit quantization. 


\section{The Convergence Analysis}\label{sec:Convergence Analysis}
In this section, we study the effect of analog aggregation transmissions and 1-bit CS on FL over the air, by analyzing its convergence behavior.
\subsection{Basic Assumptions}
To facilitate the convergence analysis, we make the following standard assumptions on the loss function and gradients.

\textbf{Assumption 1 (Lipschitz continuity, smoothness):} The gradient $\nabla F(\mathbf{w})$ of the loss function $F(\mathbf{w})$ is $L$-Lipschitz\cite{bubeck2014convex}, that is,
\begin{eqnarray}\label{eq:Lipschitz}
\|\nabla F(\mathbf{w}_{t+1})-\nabla F(\mathbf{w}_{t})\|\leq L\|\mathbf{w}_{t+1}-\mathbf{w}_{t}\|,
\end{eqnarray}
where $L$ is a non-negative Lipschitz constant for the continuously differentiable function $F(\cdot)$.


\textbf{Assumption 2 (twice-continuously differentiable):} The function $F(\mathbf{w})$ is twice-continuously differentiable and $L$-smoothness. Accordingly, the eigenvalues of the Hessian matrix of $F(\mathbf{w})$ are bounded by\cite{bubeck2014convex}:
\begin{eqnarray}\label{eq:twice-continuously differentiable}
 \nabla^2 F(\mathbf{w}_{t})\preceq L\mathbf{I}.
\end{eqnarray}

\textbf{Assumption 3 (sample-wise gradient bounded):} The sample-wise gradients at local workers are bounded by their global counterpart\cite{Bertsekas1996Neuro,friedlander2012hybrid}
\begin{eqnarray}\label{eq:bound}
\parallel \nabla f(\mathbf{w}_{t})\parallel^2 \leq \rho_1+\rho_2\parallel \nabla F(\mathbf{w}_{t})\parallel^2 ,
\end{eqnarray}
where $\rho_1\geq0$ and $0\leq\rho_2< 1$.  

\textbf{Assumption 4 (local gradient bounded ):}
The local gradients are bounded by\cite{stich2018sparsified}
\begin{eqnarray}\label{eq:gdbound}
\|\mathbf{g}_{i,t}\|^2\leq G^2, \forall i,t,
\end{eqnarray}
where $G$ is positive constant.

\subsection{Convergence Analysis}
We first analyze the total error between the recovered averaged gradient in \eqref{eq:sharedupdate1} and the ideal one in \eqref{eq:sharedupdate0}, including the errors caused by sparsification, quantization, AWGN and reconstruction algorithms.
Based on the above \textbf{Assumption 4}, we derive the following \textbf{Lemma \ref{Lemma1}} to describe the total error.
\begin{lemma}\label{Lemma1}
The total error $\mathbf{e}_{t}=\hat{\mathbf{g}}_t-\mathbf{g}_t$ at the $t$-th iteration in FL is bounded by
\begin{align}
\mathbb{E}\|\mathbf{e}_{t}\|^2=\mathbb{E}(\|\hat{\mathbf{g}}_t-\mathbf{g}_t\|^2)\leq& C^2 \left(1+(1+\delta)\frac{D-\kappa}{SD}G^2
  +\frac{\sigma^2}{\left(\sum_{i=1}^U K_i\beta_{i,t}b_{t}\right)^{2}}\right)\nonumber\\ \label{eq:Lemma1}&+ \sum_{i=1}^U\beta_{i,t}(1+\delta)\frac{D-\kappa}{D}G^2,
\end{align}
where $0<\delta<1$ is the constant in the RIP condition, $C=\frac{2\varpi}{1-\varrho}$, $\varpi=\frac{2\sqrt{1+\delta}}{\sqrt{1-\delta}}$ and $\varrho=\frac{\sqrt{2}\delta}{1-\delta}$.
\end{lemma}
\begin{proof}
The proof of \textbf{Lemma \ref{Lemma1}} is provide in Appendix \ref{Appendix A}.
\end{proof}
\begin{remark}
\textbf{Lemma \ref{Lemma1}} indicates that a larger $\kappa$ leads to a smaller error, which suggests that sparsification is applied at the expense of accuracy. And a larger $S$ leads to a smaller error because of less compression.
\end{remark}

Next we present the main theorem for the expected convergence rate of the 1-bit CS based FL over the air with analog aggregation, as in \textbf{Theorem \ref{Theorem1}}.



\begin{theorem}\label{Theorem1}
Given the power scaling factor $b_{t}$, worker selection vectors $\beta_{i,t}$, and the learning rate $\alpha = \frac{1}{L}$, we have the following convergence rate at the $T$-th iteration.  
\begin{align}\label{eq:Theorem1}
\frac{1}{T}\sum_{t=1}^{T}\parallel\nabla F(\mathbf{w}_{t-1})\parallel^2
\leq& \frac{2L}{T(1-\rho_2)}\mathbb{E}[F(\mathbf{w}_0)-F(\mathbf{w}^*)] +\frac{2L}{T(1-\rho_2)}\sum_{t=1}^{T}B_t,
\end{align}
where
\begin{align}\label{eq:BtTheo}
B_{t}=&\frac{\sum_{i=1}^U K_i\rho_1(1-\beta_{i,t})}{2LK}+\frac{C^2 }{2L} \left(1+(1+\delta)\frac{D-\kappa}{SD}G^2
  +\frac{\sigma^2}{\left(\sum_{i=1}^U K_i\beta_{i,t}b_{t}\right)^{2}}\right)\nonumber\\& + \sum_{i=1}^U\beta_{i,t}(1+\delta)\frac{D-\kappa}{2LD}G^2,
\end{align}
and $\mathbf{w}_t$ converges to $\mathbf{w}^*$.
\end{theorem}
\begin{proof}
The proof of \textbf{Theorem \ref{Theorem1}} is provide in Appendix \ref{Appendix B}.
\end{proof}

 In \textbf{Theorem \ref{Theorem1}}, the expected gradient norm is used as an indicator of convergence \cite{wang2018cooperative}. That is, the FL algorithm achieves an $\tau$-suboptimal solution if: 
\begin{align}\label{eq:indicator}
\frac{1}{T}\sum_{t=1}^{T}\frac{}{}\parallel\nabla F(\mathbf{w}_{t-1})\parallel^2
\leq& \tau,
\end{align}
which guarantees the convergence of the algorithm to a stationary point. If the objective function $F(\mathbf{w})$ is non-convex, then FL may converge to a local minimum or saddle point.

From \textbf{Theorem \ref{Theorem1}}, we have
\begin{align}
\frac{1}{T}\sum_{t=1}^{T}\frac{}{}\parallel\nabla F(\mathbf{w}_{t-1})\parallel^2
\leq& \frac{2L}{T(1-\rho_2)}\mathbb{E}[F(\mathbf{w}_0)-F(\mathbf{w}^*)] +\frac{2L}{T(1-\rho_2)}\sum_{t=1}^{T}B_t \nonumber\\ \label{eq:errorfloor}
\overset{T\rightarrow\infty}{\longrightarrow}
& \frac{2L}{T(1-\rho_2)}\sum_{t=1}^{T}B_t.
\end{align}

The error floor at convergence is given by \eqref{eq:errorfloor}. Obviously, minimizing this error floor can improve the convergence performance of FL. Capitalizing on this theoretical result, we provide a joint optimization of communication and computation next.

\section{Minimization of the Error Floor for Federated Learning Algorithm}\label{sec:Joint optimization}
In this section, we formulate a joint optimization problem to minimize the error floor in \eqref{eq:errorfloor} for 1-bit CS based FL over the air. In solving such a problem, we first develop an optimal solution via discrete programming, and then propose a computationally scalable ADMM-based suboptimal solution for large-scale wireless networks.

\subsection{Joint Optimization Problem Formulation}
In the deployment of FL over the air, the error floor in \eqref{eq:errorfloor} is accumulated over iterations, resulting a performance gap between $F(\mathbf{w}_{t-1})$ and $F(\mathbf{w}^*)$. Thus, we design an online policy to minimize this gap at each iteration, which amounts to iteratively minimizing $B_t$ under the constraint of transmit power limitation in \eqref{power_limitation}. Minimizing $B_t$ is equivalent to minimizing $R_{t}=2L B_t$, i.e.,
\begin{align}
R_{t}=&\frac{\sum_{i=1}^U K_i\rho_1(1-\beta_{i,t})}{K}+C^2 \left(1+(1+\delta)\frac{D-\kappa}{DS}G^2
  +(\sum_{i=1}^U K_i\beta_{i,t}b_{t})^{-2}\sigma^2\right) \nonumber\\\label{eq:RtTheoop}&+ \sum_{i=1}^U\beta_{i,t}(1+\delta)\frac{D-\kappa}{D}G^2.
\end{align}

At each iteration $t$, the PS aims to determine the power scaling factor $b_t$ and the scheduling indicator $\bm{\beta}_{t}=[\beta_{1,t},\beta_{2,t},...,\beta_{U,t}]$ in order to minimize $R_t$, for given values of the factors (i.e., $C$, $S$, and $\kappa$) related to 1-bit CS.  Such a joint optimization problem is formulated as
\begin{subequations}
\begin{align}\label{IterationOpt}
\textbf{P2:}\quad \min_{b_{t},\bm{\beta}_{t}} \quad & R_{t}\\ \label{con:pmax}
\text{s.t.} \quad& \  \frac{\beta^2_{i,t}K^2_ib^2_{t}}{h^2_{i,t}}\leq P_i^{\text{Max}}, \\
\quad& \ {\beta}_{i,t}\in\{0,1\}, i\in \{1,2,...,U\}.
\end{align}
\end{subequations}
\subsection{Optimal Solution via Discrete Programming}
As a mixed integer programming (MIP), \textbf{P2} is non-convex and challenging to solve due to the coupling of the power scaling factor $b_t$ and the scheduling indicator $\bm{\beta}_{t}$.
Note that once $\bm{\beta}_{t}$ is given, the problem \textbf{P2} reduces to a convex problem, where the optimal power scaling $b_{t}$ can be efficiently solved using off-the-shelf optimization algorithms, e.g., interior point method \cite{boyd2004convex}. Accordingly, a straightforward method is to enumerate all the $2^{U}$ possibilities of $\bm{\beta}_{t}$ and output the one that yields the lowest objective value. This enumeration-based method is summarized in \textbf{Algorithm \ref{alg:enumeration}}.
\begin{algorithm}[htb]
	\caption{Optimal solution via the enumeration-based method}
	\label{alg:enumeration}
	\begin{algorithmic}[1]
\renewcommand{\algorithmicrequire}{\textbf{Initialization:}}
		\REQUIRE ~~\\
		$\{P_{i}^{\text{Max}}, h_{i,t}, K_i\}_{i=1}^U$, $\bm{\Phi}$, $G$, $\kappa$.\\
		\ENSURE ~~\\
		The optimal solution $\{b_{t}^*, \bm{\beta}_{t}^*\}$.
\STATE \textbf{Repeat}
\STATE Select $\bm{\beta}_{t}$ from its possibility;
\STATE Given $\bm{\beta}_{t}$, solve \textbf{P2} to find $\{b_{t}\}$;
\STATE If the objective value is lower under this $\{b_{t}, \bm{\beta}_{t}\}$, then update $\{b_{t}^*, \bm{\beta}_{t}^*\}$;
\STATE  \textbf{Until} \{all the possible of $\bm{\beta}_{t}$ are enumerated\}
\RETURN $\{b_{t}^*, \bm{\beta}_{t}^*\}$.
	\end{algorithmic}
\end{algorithm}
\begin{remark}\label{enumeration-based}
The enumeration-based method may be applicable for a small number of workers, e.g., $U\leq 10$; however, it quickly becomes computationally infeasible as $U$ increases. 
\end{remark}

\subsection{ADMM-based Suboptimal Solution}
The enumeration-based method proposed in the last subsection is simple to implement, because the computation involves basic function evaluations only. However, large-scale networks with much increased searching dimensions makes it susceptible to high computational complexity.  To address the problem, we propose an ADMM-based algorithm to jointly optimize the local worker selection and power control. As we will show later, the proposed ADMM-based approach has a computational complexity that increases linearly with the network size $U$.

The main idea is to decompose the hard combinatorial optimization \textbf{P2} into $U$ parallel smaller integer programming problems. Nonetheless, conventional decomposition techniques, such as dual decomposition, cannot be directly applied to \textbf{P2} due to the coupled variables $\{b_{t}, \bm{\beta}_{t}\}$ and the constraint \eqref{con:pmax} among the workers. To eliminate these coupling factors, we first introduce an auxiliary vector $\mathbf{r}_{t}=[r_{1,t},r_{2,t},...,r_{U,t}]$ and define two auxiliary functions as
\begin{align}\label{eq:Q1}
Q_1(\mathbf{r}_{t})=C^2(\sum_{i=1}^U K_ir_{i,t})^{-2}\sigma^2,
\end{align}
and
\begin{align}
Q_2(\bm{\beta}_{t})=&\frac{\sum_{i=1}^U K_i\rho_1(1-\beta_{i,t})}{K}+C^2 \left(1+(1+\delta)\frac{D-\kappa}{SD}G^2\right)\label{eq:Q2}
   + \sum_{i=1}^U\beta_{i,t}(1+\delta)\frac{D-\kappa}{D}G^2.
\end{align}


Then we introduce another auxiliary vector $\mathbf{q}_{t}=[q_{1,t},q_{2,t},...,q_{U,t}]$ and reformulate \textbf{P2} as the following \textbf{P3}.
\begin{subequations}
\begin{align}\label{IterationOpttightenP4}
\textbf{P3:}\qquad \min_{b_{t},\{r_{i,t},q_{i,t},\beta_{i,t}\}_{i=1}^{U}} &\qquad Q_1(\mathbf{r}_{t})+Q_2(\bm{\beta}_{t})\\ \label{constraint:P3b} 
\text{s.t.} & \qquad  \bigg|\frac{K_ir_{i,t}}{h_{i,t}}\bigg|^2\leq P_i^{\text{Max}}, \\ \label{constraint:P3c}
& \qquad r_{i,t}=\beta_{i,t}q_{i,t},\\\label{constraint:P3d}
&\qquad q_{i,t}=b_{t},\\
& \qquad r_{i,t}>0, b_t>0, \\&\qquad{\beta}_{i,t}\in\{0,1\},\\&\qquad i\in \{1,2,...,U\}.
\end{align}
\end{subequations}
Here, the constraints \eqref{constraint:P3c} and \eqref{constraint:P3d} are introduced to decouple $\beta_{i,t}$ and $b_t$ while guaranteeing that \textbf{P3} and \textbf{P2} are equivalent.

 By introducing multipliers $\nu_{i,t}\geq 0$'s, $\xi_{i,t}\geq 0$'s and $\varsigma_{i,t}\geq 0$'s to the constraints in \eqref{constraint:P3b}, \eqref{constraint:P3c} and \eqref{constraint:P3d}, we can write a partial augmented Lagrangian of \textbf{P3} as
 \begin{align}
 \mathcal{L}(b_{t},\bm{\beta}_{t},\mathbf{r}_{t},\mathbf{q}_{t},\bm{\nu}_{t},\bm{\xi}_{t},\bm{\varsigma}_{t})=&Q_1(\mathbf{r}_{t})+Q_2(\bm{\beta}_{t})
+\sum_{i=1}^U\nu_{i,t}\bigg(\bigg|\frac{K_ir_{i,t}}{h_{i,t}}\bigg|^2- P_i^{\text{Max}}\bigg)\\ \nonumber
&+\sum_{i=1}^U\xi_{i,t}(r_{i,t}-\beta_{i,t}q_{i,t})
+\frac{c}{2}\sum_{i=1}^U(r_{i,t}-\beta_{i,t}q_{i,t})^2
\\ \nonumber
&+\sum_{i=1}^U\varsigma_{i,t}(q_{i,t}-b_t)+\frac{c}{2}\sum_{i=1}^U(q_{i,t}-b_t)^2,
 \end{align}
where $\bm{\nu}_{t}=[\nu_{1,t},\nu_{2,t},...,\nu_{U,t}]$, $\bm{\xi}_{t}=[\xi_{1,t},\xi_{2,t},...,\xi_{U,t}]$, $\bm{\varsigma}_{t}=[\varsigma_{1,t},\varsigma_{2,t},...,\varsigma_{U,t}]$, and $c>0$ is a fixed step size. The corresponding dual problem is
\begin{subequations}
\begin{align}\label{dualP}
\textbf{P4:}\quad \max_{\{\nu_{i,t},\xi_{i,t},\varsigma_{i,t}\}_{i=1}^U} \quad  &\mathcal{M}(\bm{\nu}_{t},\bm{\xi}_{t},\bm{\varsigma}_{t}) \\
\text{s.t.} \quad& \nu_{i,t}\geq 0,\ \xi_{i,t}\geq 0,\ \varsigma_{i,t}\geq 0,i\in \{1,2,...,U\},
\end{align}
\end{subequations}
where $\mathcal{M}(\bm{\nu}_{t},\bm{\xi}_{t},\bm{\varsigma}_{t})$ is the dual function, which is given by
\begin{subequations}
\begin{align}\label{IterationOpttightenP4}
 \mathcal{M}(\bm{\nu}_{t},\bm{\xi}_{t},\bm{\varsigma}_{t})=\min_{b_{t},\{r_{i,t},q_{i,t},\beta_{i,t}\}_{i=1}^{U}} \quad & \mathcal{L}(b_{t},\mathbf{r}_{t},\mathbf{q}_{t},\bm{\beta}_{t}) \\ \label{constraint:P5b} 
\text{s.t.} \quad & \ r_{i,t}>0, b_t>0, q_{i,t}>0, \\ \quad &\beta_{i,t}\in\{0,1\}, i\in \{1,2,...,U\}.
\end{align}
\end{subequations}

The ADMM technique \cite{boyd2011distributed} solves the dual problem \textbf{P4} by iteratively updating $\{\mathbf{r}_{t},b_{t}\}$, $\{\mathbf{q}_{t},\bm{\beta}_{t}\}$, and $\{\bm{\nu}_{t},\bm{\xi}_{t},\bm{\varsigma}_{t}\}$. We denote the values at the $l$-th iteration as $\{\mathbf{r}_{t}^{\{l\}},b_{t}^{\{l\}}\}$, $\{\mathbf{q}_{t}^{\{l\}},\bm{\beta}_{t}^{\{l\}}\}$, and $\{\bm{\nu}_{t}^{\{l\}},\bm{\xi}_{t}^{\{l\}},\bm{\varsigma}_{t}^{\{l\}}\}$. Then, the update of the variables is sequentially performed at the $(l + 1)$-th iteration as follows:

\textbf{1) Step 1:} Given $\{\mathbf{q}_{t}^{\{l\}},\bm{\beta}_{t}^{\{l\}}\}$, and $\{\bm{\nu}_{t}^{\{l\}},\bm{\xi}_{t}^{\{l\}},\bm{\varsigma}_{t}^{\{l\}}\}$, we first minimize $\mathcal{L}$ with respect to $\{\mathbf{r}_{t},b_{t}\}$, where
\begin{align}\label{Iterationrbl}
 \{\mathbf{r}_{t}^{\{l+1\}},b_{t}^{\{l+1\}}\}=\arg \min_{\mathbf{r}_{t},b_{t}} \mathcal{L}(\mathbf{r}_{t},b_{t};\mathbf{p}_{t}^{\{l\}},\bm{\beta}_{t}^{\{l\}},\bm{\nu}_{t}^{\{l\}},\bm{\xi}_{t}^{\{l\}},\bm{\varsigma}_{t}^{\{l\}}).
\end{align}

 Notice that \eqref{Iterationrbl} is a strictly convex problem, which can be easily solved to obtain the optimal solution, e.g., by using the projected Newton’s method \cite{boyd2004convex}. Since the complexity of solving this problem in \eqref{Iterationrbl} does not scale with $U$ (i.e., $\mathcal{O}(1)$ complexity), thus the overall computational complexity of \textbf{Step 1} is $\mathcal{O}(1)$.

\textbf{2) Step 2:} Given $\{\mathbf{r}_{t}^{\{l+1\}},b_{t}^{\{l+1\}}\}$, and $\{\bm{\nu}_{t}^{\{l\}},\bm{\xi}_{t}^{\{l\}},\bm{\varsigma}_{t}^{\{l\}}\}$, we then minimize $\mathcal{L}$ with respect to $\{\mathbf{q}_{t},\bm{\beta}_{t}\}$, where
\begin{align}\label{Iterationrbibetal}
 \{\mathbf{q}_{t}^{\{l+1\}},\bm{\beta}_{t}^{\{l+1\}}\}=\arg \min_{\mathbf{q}_{t},\bm{\beta}_{t}} \mathcal{L}(\mathbf{q}_{t},\bm{\beta}_{t};\mathbf{r}_{t}^{\{l+1\}},b_{t}^{\{l+1\}},\bm{\nu}_{t}^{\{l\}},\bm{\xi}_{t}^{\{l\}},\bm{\varsigma}_{t}^{\{l\}}).
\end{align}

This optimization can be decomposed into $U$ parallel subproblems. In each subproblem (e.g., $i$-th subproblem), by considering $\beta_{i,t}=0$ and $\beta_{i,t}=1$, respectively, the $i$-th subproblem is expressed as
\begin{equation}\label{Iterationrbibetal01}
 \{q_{i,t}\}^{\{l+1\}}=
 \left\{
\begin{aligned}
\arg \min_{q_{i,t}} \mathcal{L}\left(q_{i,t},0;\{\mathbf{r}_{t}\}^{\{l+1\}},\{b_{t}\}^{\{l+1\}},\{\nu_{i,t}\}^{\{l\}},\{\xi_{i,t}\}^{\{l\}},\{\varsigma_{i,t}\}^{\{l\}}\right) & , & \beta_{i,t}=0, \\
\arg \min_{q_{i,t}} \mathcal{L}\left(q_{i,t},1;\{\mathbf{r}_{t}\}^{\{l+1\}},\{b_{t}\}^{\{l+1\}},\{\nu_{i,t}\}^{\{l\}},\{\xi_{i,t}\}^{\{l\}},\{\varsigma_{i,t}\}^{\{l\}}\right) & , & \beta_{i,t}=1.
\end{aligned}
\right.
\end{equation}
where
\begin{align}
\mathcal{L}&\left(q_{i,t},0;\{\mathbf{r}_{t}\}^{\{l+1\}},\{b_{t}\}^{\{l+1\}},\{\nu_{i,t}\}^{\{l\}},\{\xi_{i,t}\}^{\{l\}},\{\varsigma_{i,t}\}^{\{l\}}\right)\\ \nonumber =&\frac{K_i\rho_1}{K}
+\{\xi_{i,t}\}^{\{l\}}\{r_{i,t}\}^{\{l+1\}}
+\frac{c}{2}\left(\{r_{i,t}\}^{\{l+1\}}\right)^2
+\varsigma_{i,t}\left(q_{i,t}-\{b_{t}\}^{\{l+1\}}\right)+\frac{c}{2}\left(q_{i,t}-\{b_{t}\}^{\{l+1\}}\right)^2,
\end{align}
and
\begin{align}
\mathcal{L}&\left(q_{i,t},1;\{\mathbf{r}_{t}\}^{\{l+1\}},\{b_{t}\}^{\{l+1\}},\{\nu_{i,t}\}^{\{l\}},\{\xi_{i,t}\}^{\{l\}},\{\varsigma_{i,t}\}^{\{l\}}\right)\nonumber\\ \nonumber =&(1+\delta)\frac{D-\kappa}{D}G^2+
\{\xi_{i,t}\}^{\{l\}}\left(\{r_{i,t}\}^{\{l+1\}}-q_{i,t}\right)
+\frac{c}{2}\left(\{r_{i,t}\}^{\{l+1\}}-q_{i,t}\right)^2 \nonumber\\
&+\varsigma_{i,t}\left(q_{i,t}-\{b_{t}\}^{\{l+1\}}\right)+\frac{c}{2}\left(q_{i,t}-\{b_{t}\}^{\{l+1\}}\right)^2.
\end{align}

 For both $\beta_{i,t}=0$ and $\beta_{i,t}=1$, \eqref{Iterationrbibetal01} solves a strictly convex problem, and hence is easy to obtain the optimal solution. Accordingly, we can simply select between $\beta_{i,t}=0$ or $\beta_{i,t}=1$ that yields a smaller objective value in \eqref{Iterationrbibetal01} as $\{\beta_{i,t}\}^{\{l+1\}}$, and the corresponding optimal solution of $\{q_{i,t}\}^{\{l+1\}}$.  After solving the $U$ parallel subproblems, the optimal solution to \eqref{Iterationrbibetal} is given by $\{\mathbf{q}_{t}^{\{l+1\}},\bm{\beta}_{t}^{\{l+1\}}\}$. Notice that the complexity of solving each subproblem in \eqref{Iterationrbibetal} scales with $U$, and thus the overall computational complexity of \textbf{Step 2} is $\mathcal{O}(U)$.

 \textbf{3) Step 3:} Finally, given $\{\mathbf{r}_{t}^{\{l+1\}},b_{t}^{\{l+1\}}\}$ and $\{\mathbf{q}_{t}^{\{l+1\}},\bm{\beta}_{t}^{\{l+1\}}\}$, we maximize $\mathcal{L}$ with respect to $\{\bm{\nu}_t,\bm{\xi}_t,\bm{\varsigma}_t\}$, which is achieved by updating the multipliers as follows
 \begin{align}\label{nu}
\{\nu_{i,t}\}^{\{l+1\}}= \{\nu_{i,t}\}^{\{l\}}+c\left(\bigg|\frac{K_i\{r_{i,t}\}^{\{l+1\}}}{h_{i,t}}\bigg|^2- P_i^{\text{Max}}\right), \quad i=1,...,U,
\end{align}
 \begin{align}\label{xi}
\{\xi_{i,t}\}^{\{l+1\}}= \{\xi_{i,t}\}^{\{l\}}+c\left(\{r_{i,t}\}^{\{l+1\}}-\{\beta_{i,t}\}^{\{l+1\}}\{q_{i,t}\}^{\{l+1\}}\right), \quad i=1,...,U,
\end{align}
 \begin{align}\label{varsigma}
\{\varsigma_{i,t}\}^{\{l+1\}}= \{\varsigma_{i,t}\}^{\{l\}}+c\left(\{q_{i,t}\}^{\{l+1\}}-\{b_t\}^{\{l+1\}}\right), \quad i=1,...,U.
\end{align}

Obviously, the computational complexity of \textbf{Step 3} is $\mathcal{O}(U)$ as well.

The ADMM method implements the above \textbf{Steps 1} to \textbf{3} iteratively until meeting a specified stopping criterion. In general, the stopping criterion is specified by two thresholds \cite{boyd2011distributed}: an absolute tolerance (e.g., $\sum_{i=1}^U |\{q_{i,t}\}^{\{l+1\}}-\{b_{t}\}^{\{l+1\}}|$) and a relative tolerance (e.g., $ |\{b_{t}\}^{\{l+1\}}-\{b_{t}\}^{\{l\}}|$). The pseudo-code of the ADMM based method solving (\textbf{P3}) is summarized in \textbf{Algorithm \ref{alg:ADMM}}.
\begin{algorithm}[htb]
	\caption{ADMM-based suboptimal solution}
	\label{alg:ADMM}
	\begin{algorithmic}[1]
\renewcommand{\algorithmicrequire}{\textbf{Initialization:}}
		\REQUIRE ~~\\
		$\{P_{i}^{\text{Max}}, h_{i,t}, K_i\}_{i=1}^U$, $\bm{\Phi}$, $G$, $\kappa$.\\
		\ENSURE ~~\\
		The optimal solution $\{b_{t}^*, \bm{\beta}^*_{t}\}$.
\STATE \textbf{Repeat}
\STATE Update $\{\mathbf{r}_{t}^{\{l+1\}},b_{t}^{\{l+1\}}\}$ by solving \eqref{Iterationrbl};
\STATE Update $\{\mathbf{q}_{t}^{\{l+1\}},\bm{\beta}_{t}^{\{l+1\}}\}$ by solving \eqref{Iterationrbibetal};
\STATE Update $\{\bm{\nu}_t^{\{l+1\}},\bm{\xi}_t^{\{l+1\}},\bm{\varsigma}_t^{\{l+1\}}\}$ by using \eqref{nu}, \eqref{xi}, and\eqref{varsigma};
\STATE  \textbf{Until} \{the convergence threshold is satisfied or the maximum number of iterations is reached\}.
\RETURN $\{b_{t}^*, \bm{\beta}^*_{t}\}$.
\end{algorithmic}
\end{algorithm}

\begin{remark}\label{remark:convergence}
The proposed \textbf{Algorithm \ref{alg:ADMM}} is guaranteed to converge, because the dual problem \textbf{P4} is convex. Its convergence is insensitive to the step size $c$ \cite{ghadimi2014optimal}. Due to the potential duality gap of non-convex problems, \textbf{Algorithm \ref{alg:ADMM}} may not exactly converge to the primal optimal solution to \textbf{P3},  Thus, the dual optimal solution $\{b_{t}^*, \bm{\beta}^*_{t}\}$ is an approximate solution to \textbf{P3}. 
\end{remark}

\begin{remark}\label{remark:complexity}
We deduce that the computational complexity of one ADMM iteration (including the 3 steps) is $\mathcal{O}(U)$, because the highest complexity of these three steps is $\mathcal{O}(U)$. This complexity $\mathcal{O}(U)$ is less sensitive to $U$ than the complexity $\mathcal{O}(2^{U})$ in the enumeration-based method.
\end{remark}

\section{Simulation Results And Evaluation} \label{Sec:Numerical Results}
In the simulations, we evaluate the performance of the proposed 1-bit CS based FL over the air for an image classification task. The simulation settings are given as follows unless specified otherwise.
We consider that the FL system has $U=10$ workers, and set their maximum peak power to be $P^{\texttt{Max}}_i=P^{\texttt{Max}}= 10$ mW for any $i \in [1, U]$. The wireless channels between the workers and the PS are modeled as i.i.d. Rayleigh fading, by generating $h_{i,t}$'s from an normal distribution $\mathcal{N}(0, 1)$ for different $i$ and $t$. Without loss of the generality, the variance of AWGN at PS is set to be $\sigma^2=10^{-4}$ mW, i.e., $SNR=\frac{P^{\text{Max}}}{\sigma^2}=5$ dB. We perform top $\kappa=10$ sparsification, and the dimension of compressed local $\mathcal{C}(\mathbf{g}_i)$'s is set to $S=1000$. The elements of the measurement matrix $\bm{\Phi}$ are generated from $\mathcal{N}(0, 1/S)$. The BIHT algorithm in \cite{jacques2013robust} is selected for the signal reconstruction at the PS.

We consider the learning task of handwritten-digit recognition using the well-known MNIST dataset\footnote{http://yann.lecun.com/exdb/mnist/} that consists of 10 classes ranging from digit ``0" to ``9". In the MNIST dataset, a total of 60000 labeled training data samples and 10000 test samples are available for training a learning model. In our experiments, we train a multilayer perceptron (MLP) with a 784-neuron input layer, a 64-neuron hidden layer, and a 10-neuron softmax output layer. We adopt cross entropy as the loss function, and rectified linear unit (ReLU) as the activation function. The total number of parameters in the MLP is $D=50890$. The learning rate $\alpha$ is set as 0.1. We randomly select $3000$ distinct training samples and distribute them to all local workers as their different local datasets, i.e., $K_i=\bar{K}=3000$, for any $i \in [1, U]$.

For performance evaluation, we provide the results of training loss and test accuracy versus communication rounds under different parameter settings as follows. 

\begin{figure}[tb]
  \centering
  \subfigure[Training loss]{\includegraphics[width=0.45\textwidth]{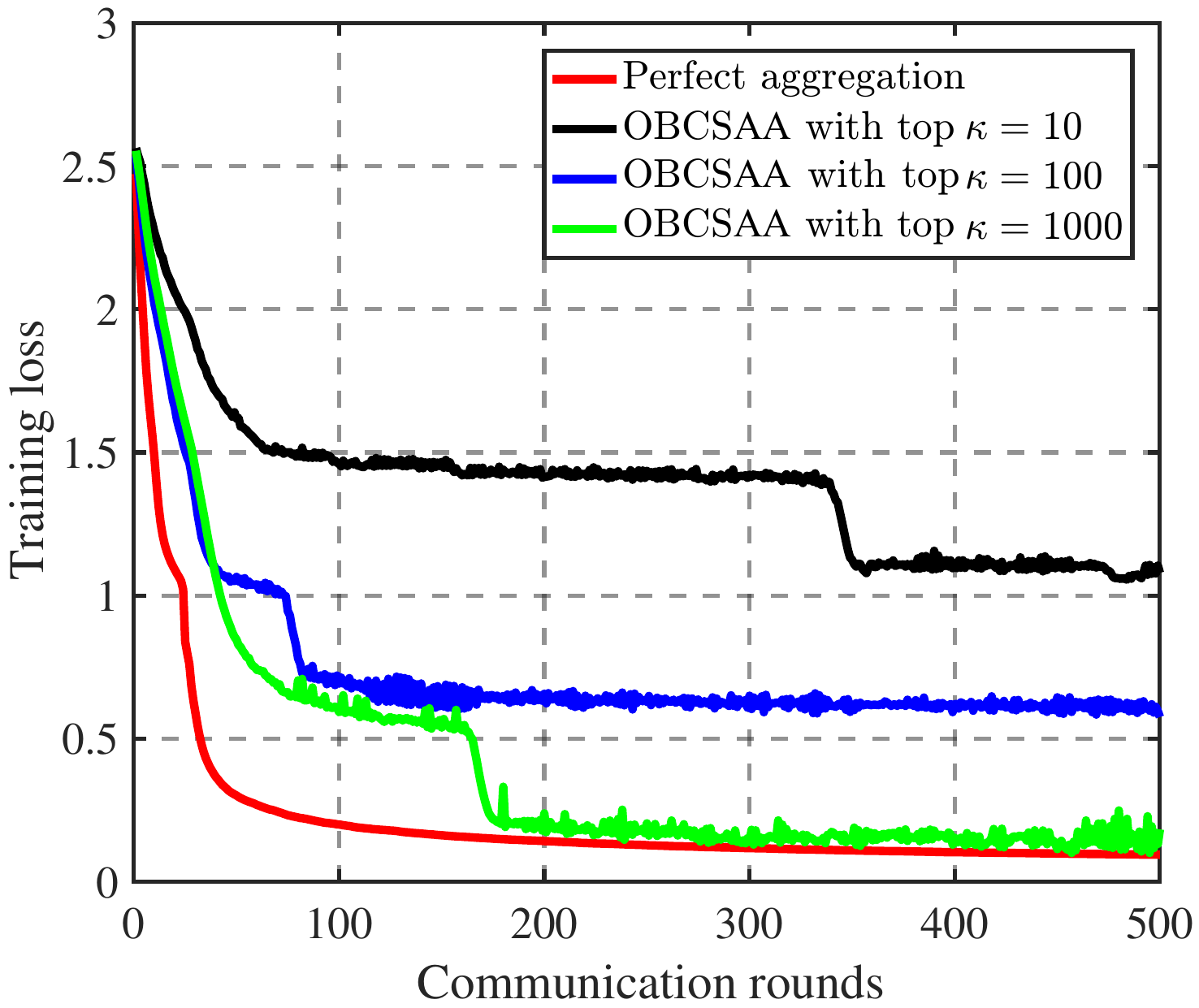}
  \label{fig:Loss-topk} }
  \subfigure[Test accuracy]{\includegraphics[width=0.45\textwidth]{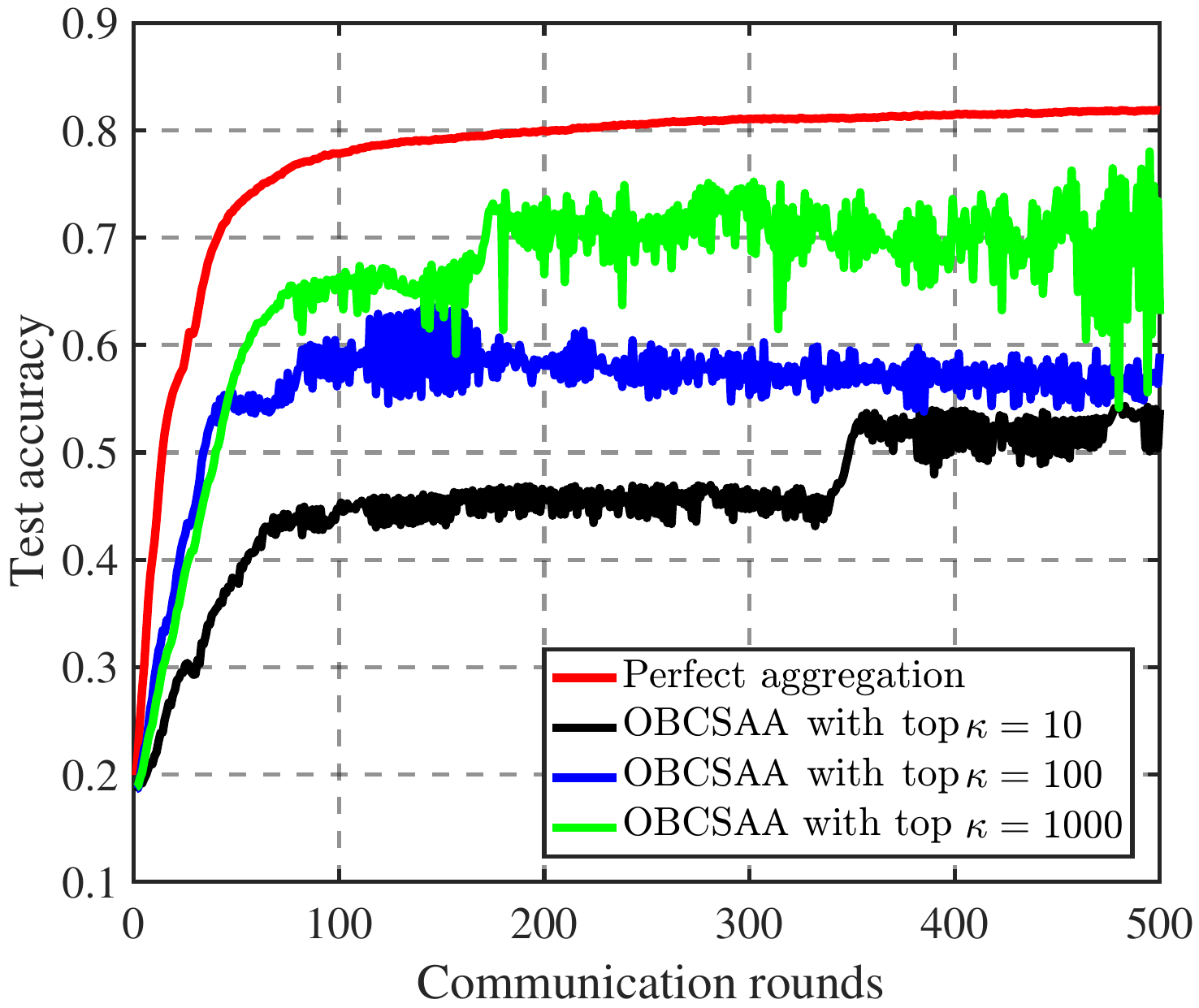}\label{fig:Accuracy-topk}}
  \caption{The performance of our proposed OBCSAA under different sparsification operators compared to perfect aggregation without sparsification.}\label{fig:topk}
\end{figure}

In Fig. \ref{fig:topk}, we first explore the impact of different sparsification operators on our proposed OBCSAA by evaluating the training loss and test accuracy of the MLP. For comparison, we use a benchmark where the transmission of local gradient updates is always reliable and error-free to achieve perfect aggregation, i.e., overlooking the influence of the wireless channel. This benchmark is an ideal case, which is named as \emph{perfect aggregation}. To satisfy RIP condition, $S$ is set to $10000$. It is observed that our proposed OBCSAA can provide desired performance (which approaches to that of \emph{perfect aggregation}), with a degree of sparsification, e.g., $\kappa=1000$, where the sparsity ratio is $1000/50890$. As $\kappa$ increases, when all FL algorithms converge, the training loss decreases and the test accuracy increases. This is because that the larger $\kappa$ is, the less gradient update information loses per communication round.

\begin{figure}[tb]
  \centering
  \subfigure[Training loss]{\includegraphics[width=0.45\textwidth]{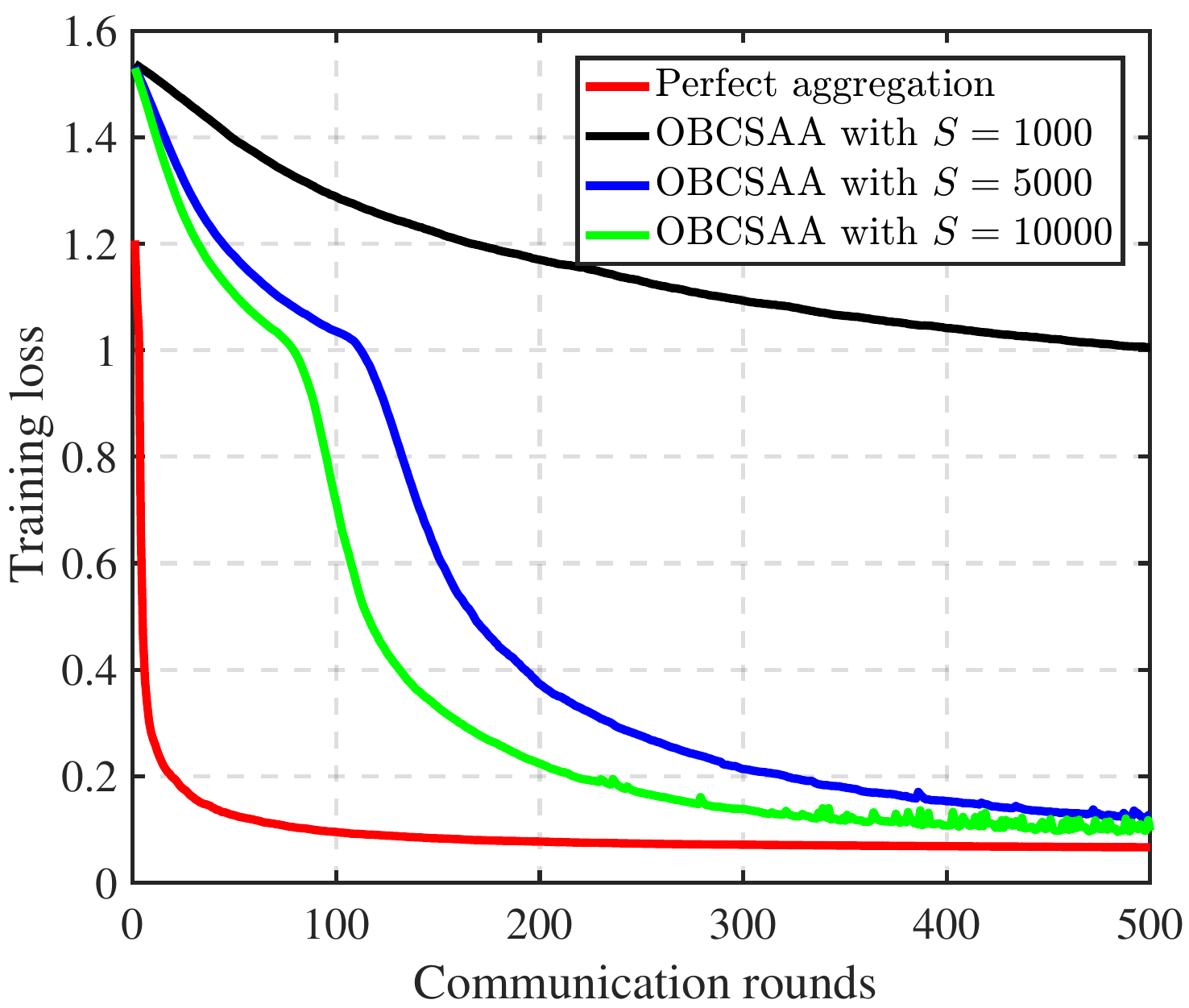}
  \label{fig:Loss_S} }
  \subfigure[Test accuracy]{\includegraphics[width=0.45\textwidth]{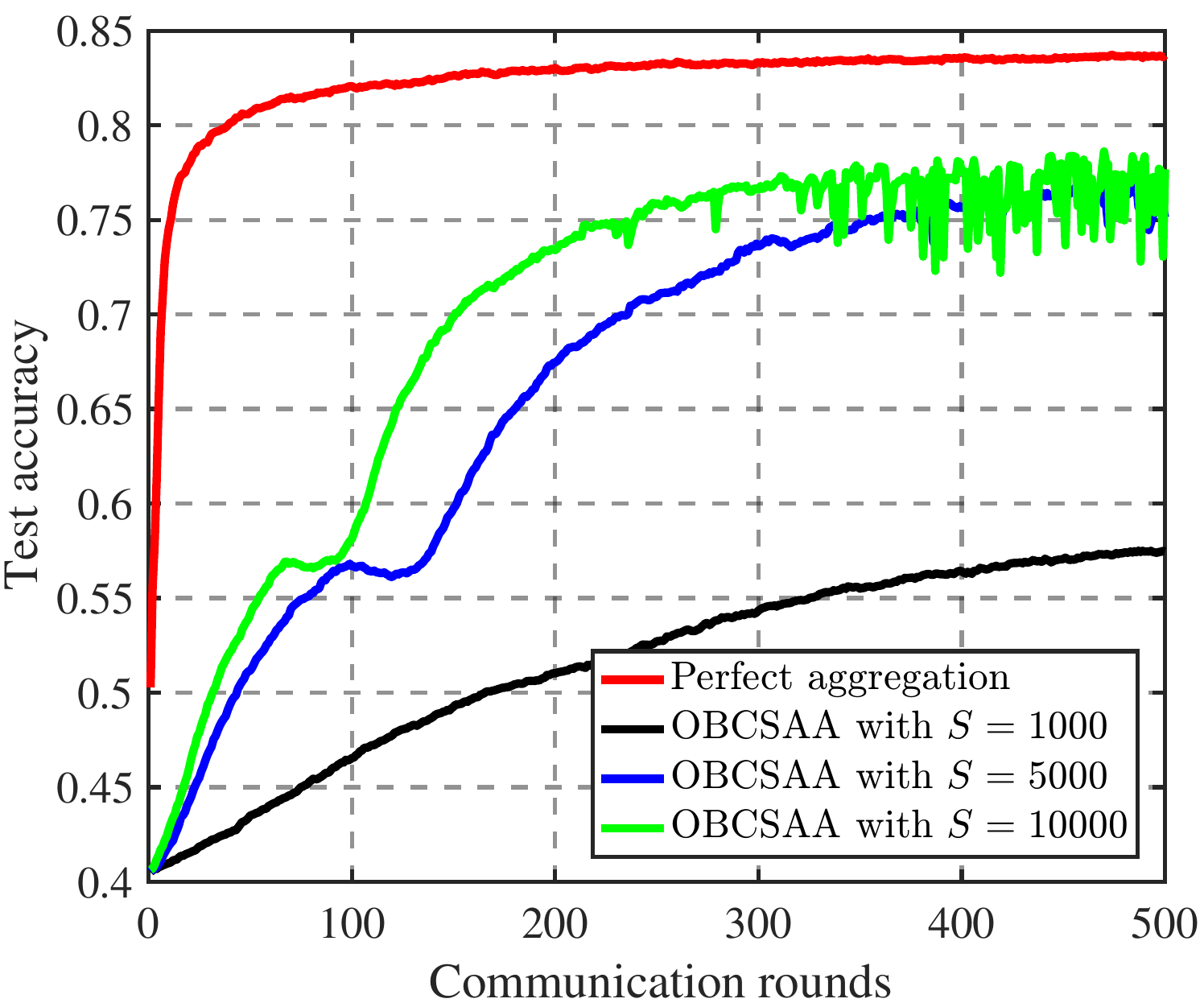}\label{fig:Accuracy_S}}
  \caption{The performance of our proposed OBCSAA under different $S$.}\label{fig:S}
\end{figure}

Fig. \ref{fig:S} shows the impact of the reduced dimension size $S$ on the performance of our proposed OBCSAA under $\kappa=1000$, where the performance increases as $S$ increases. When $S$ is large enough, performance barely increases. This is because that the larger $S$ is, the more conducive to signal reconstruction. When $S$ is large enough, the optimal performance of the reconstruction algorithm is achieved. In fact, the larger $S$ is, the more communication resources are needed. Thus, there is a tradeoff between FL performance and communication efficiency. Compared with the traditional uncompressed FL adopting digital communications, our proposed OBCSAA under $S=5000$ and $\kappa=1000$ occupies only one channel and $\frac{5000}{50890}$ transmission time, while the performance is less than $10$ percent lower than that of \emph{perfect aggregation}. These results illustrates  that our OBCSAA under appropriate parameters can greatly reduce the communication overhead and transmission latency while ensuring considerable FL performance.

\begin{figure}[tb]
  \centering
  \subfigure[Training loss]{\includegraphics[width=0.45\textwidth]{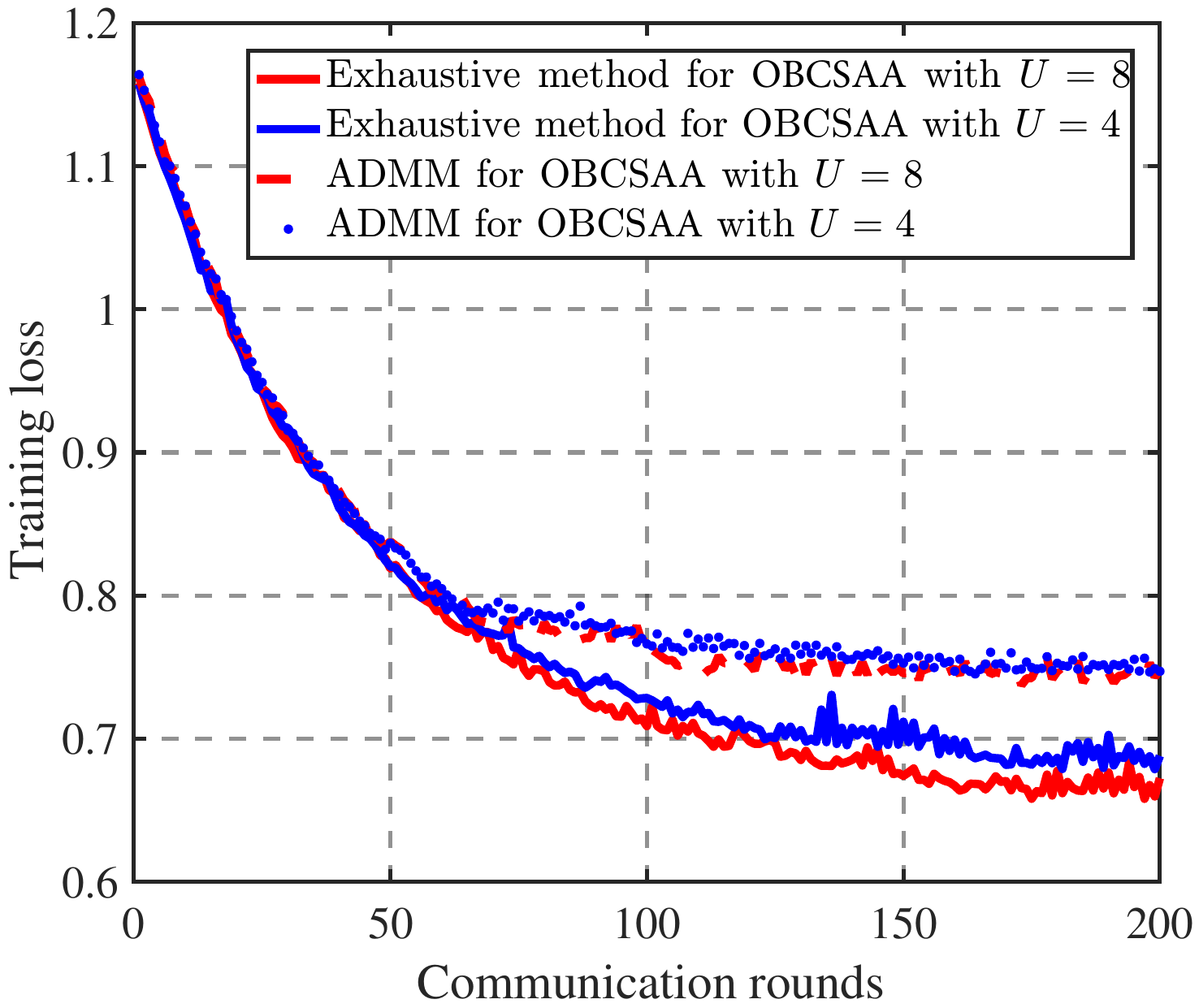}
  \label{fig:Loss_ADMM1} }
  \subfigure[Test accuracy]{\includegraphics[width=0.45\textwidth]{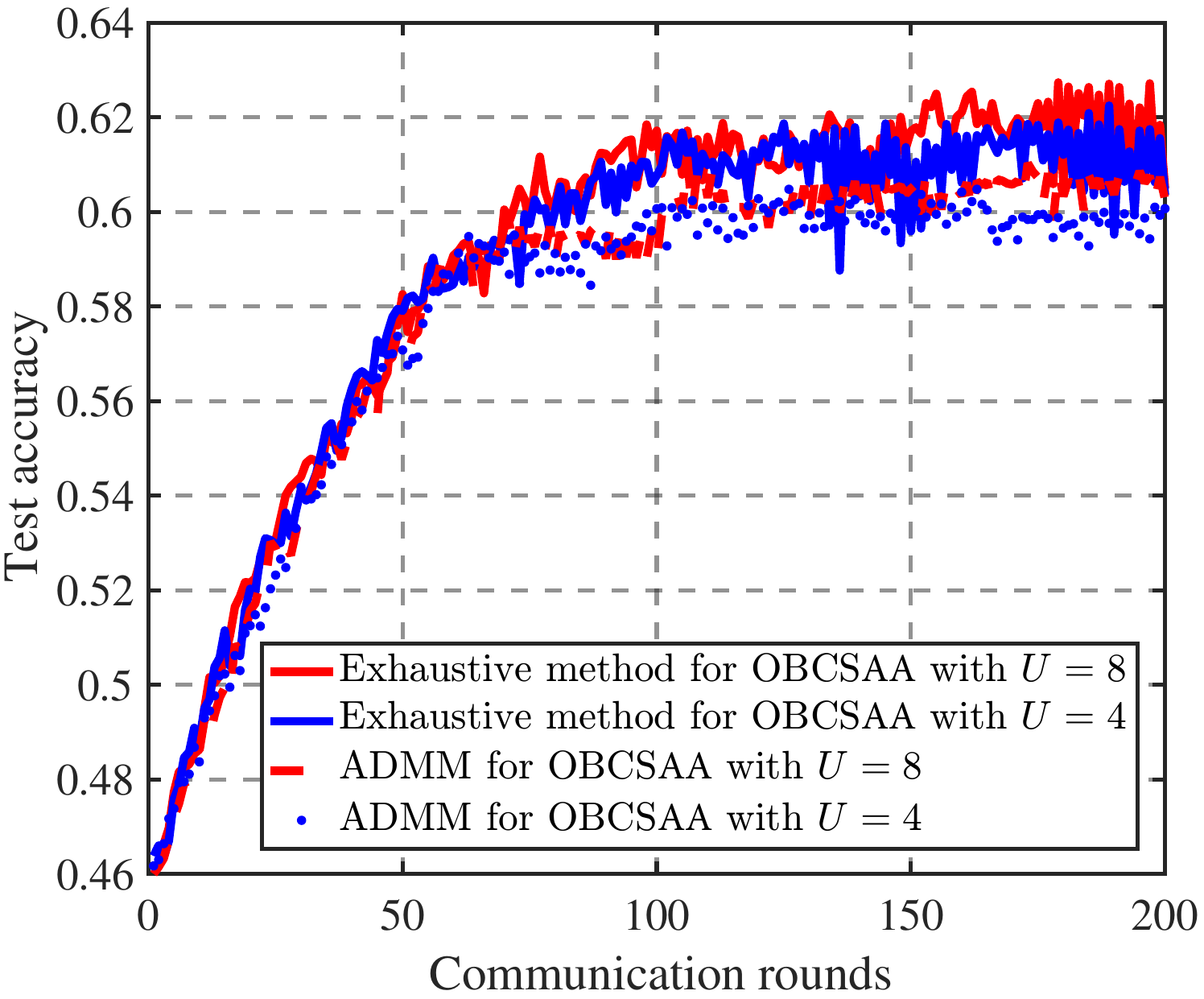}\label{fig:Accuracy_ADMM1}}
  \caption{The performance of joint optimization solving methods for our proposed OBCSAA under different $U$.}\label{fig:ADMM1}
\end{figure}

The performance of the proposed enumeration-based method and ADMM for OBCSAA under different $U$ are compared in Fig. \ref{fig:ADMM1}, where the enumeration-based method has better performance compared to ADMM. This results precisely demonstrate the effectiveness of our joint optimization scheme, which can alleviate the impact of aggregation errors on FL. Besides, we can see that the performance is higher, when the total number of local workers $U$ is larger. This is because an increase in the number of workers leads to an increased volume of data available for the FL algorithm and more workers with high channel gain can be selected.

\begin{figure}[tb]
  \centering
  \subfigure[Training loss]{\includegraphics[width=0.45\textwidth]{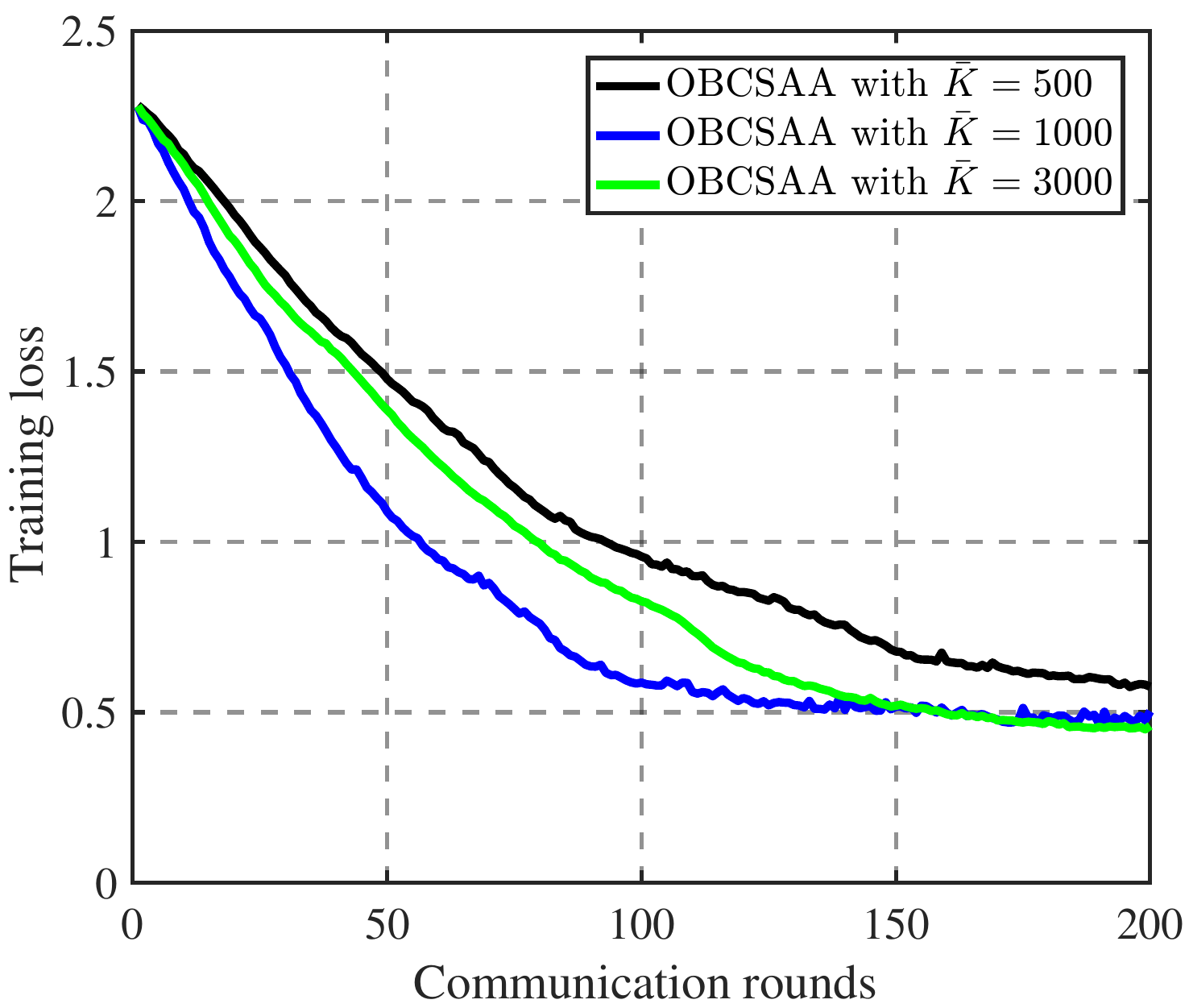}
  \label{fig:Loss_K} }
  \subfigure[Test accuracy]{\includegraphics[width=0.45\textwidth]{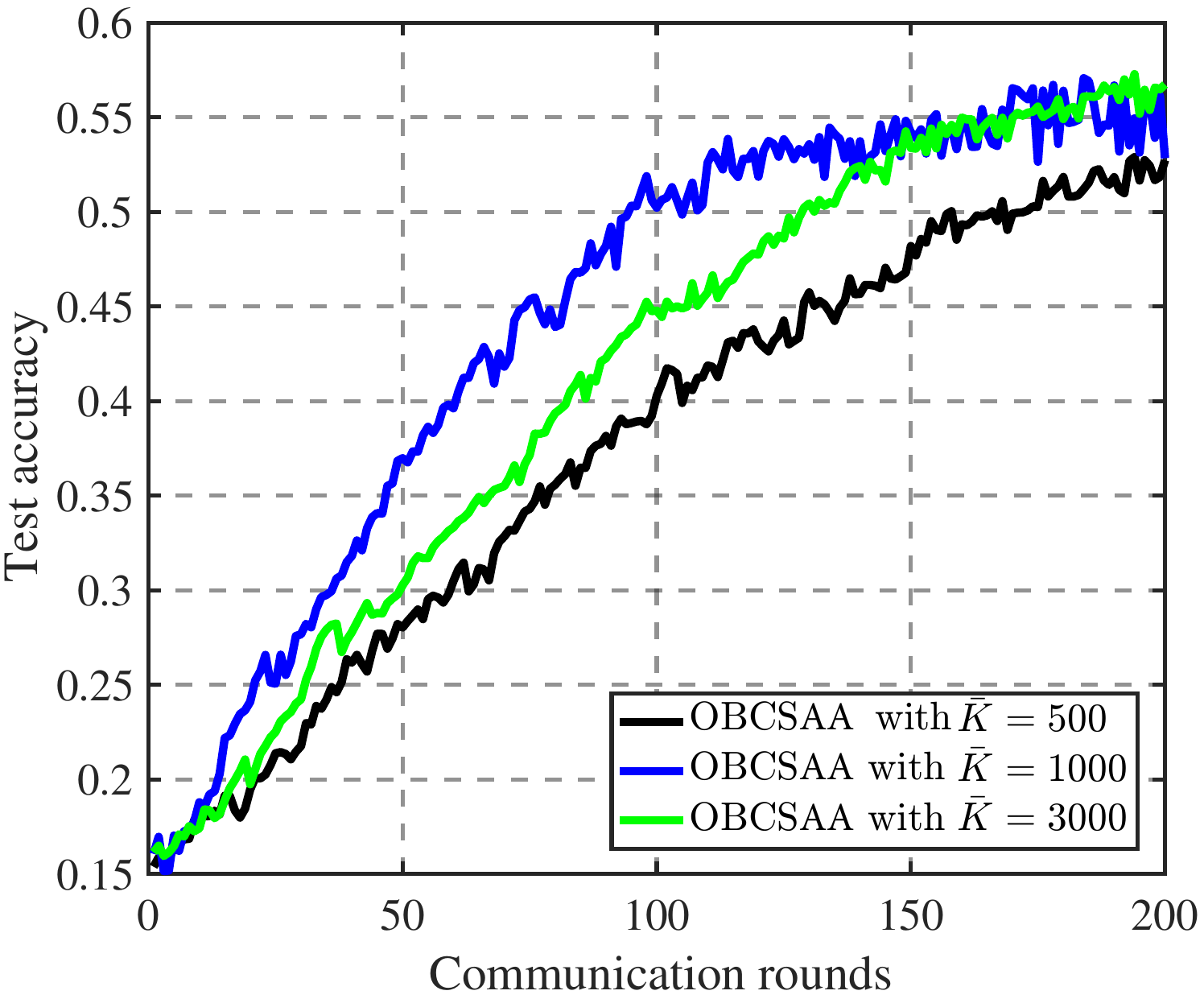}\label{fig:Accuracy_K}}
  \caption{The performance of our proposed OBCSAA under different $\bar{K}$.}\label{fig:K}
\end{figure}

Fig. \ref{fig:K} presents the impact of the number of data samples per worker $\bar{K}$ on our proposed OBCSAA. In this figure, the performance improves as $\bar{K}$ increases. When $\bar{K}$ is large enough, the performance barely improves. This is because that as $\bar{K}$ increases, the PS has more data samples for training and hence has higher performance. As $\bar{K}$ continues to increase, the improvement on learning accuracy becomes trivial when the PS already has enough data samples for training.

\begin{figure}[tb]
  \centering
  \subfigure[Training loss]{\includegraphics[width=0.45\textwidth]{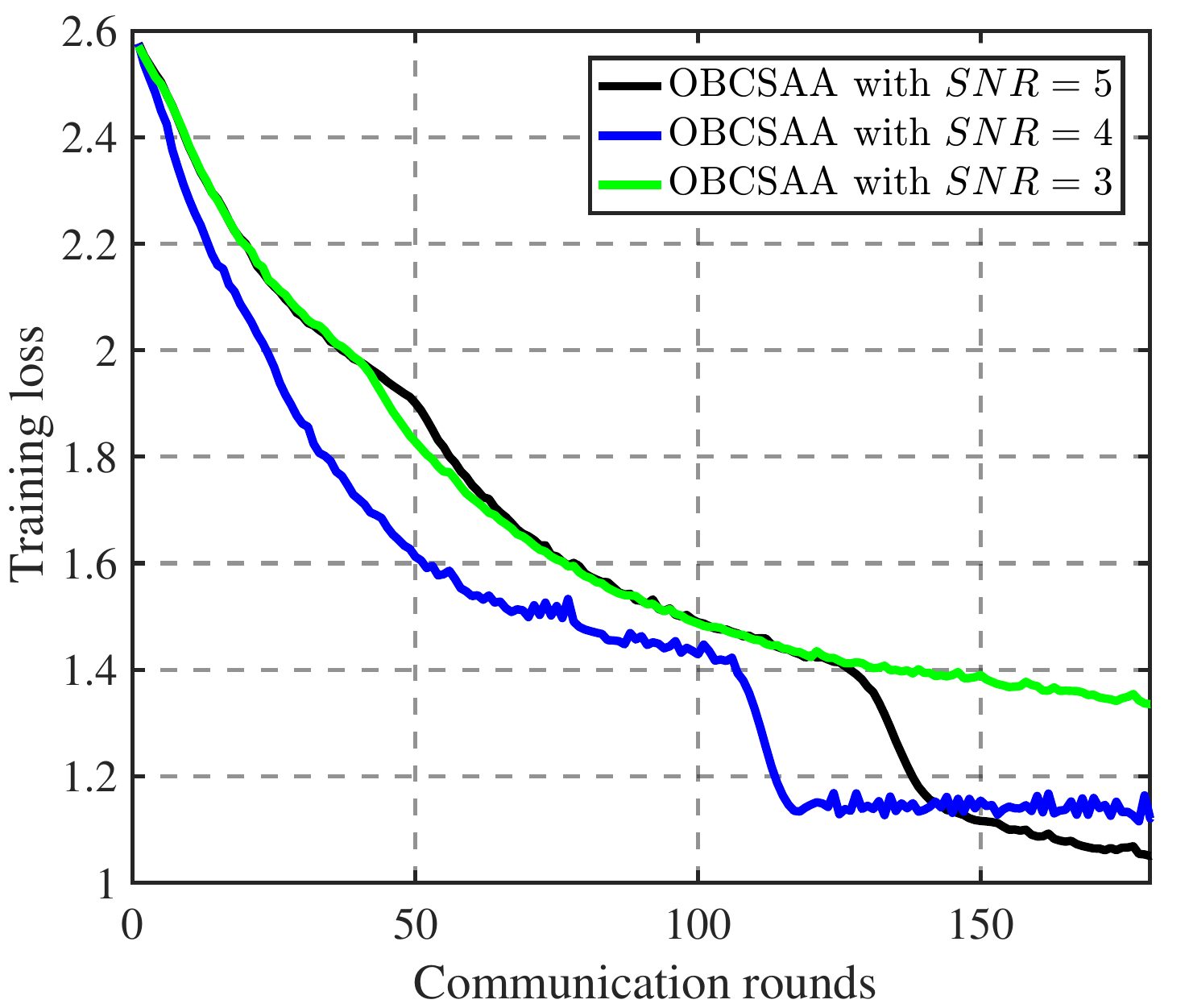}
  \label{fig:Loss_sigma} }
  \subfigure[Test accuracy]{\includegraphics[width=0.455\textwidth]{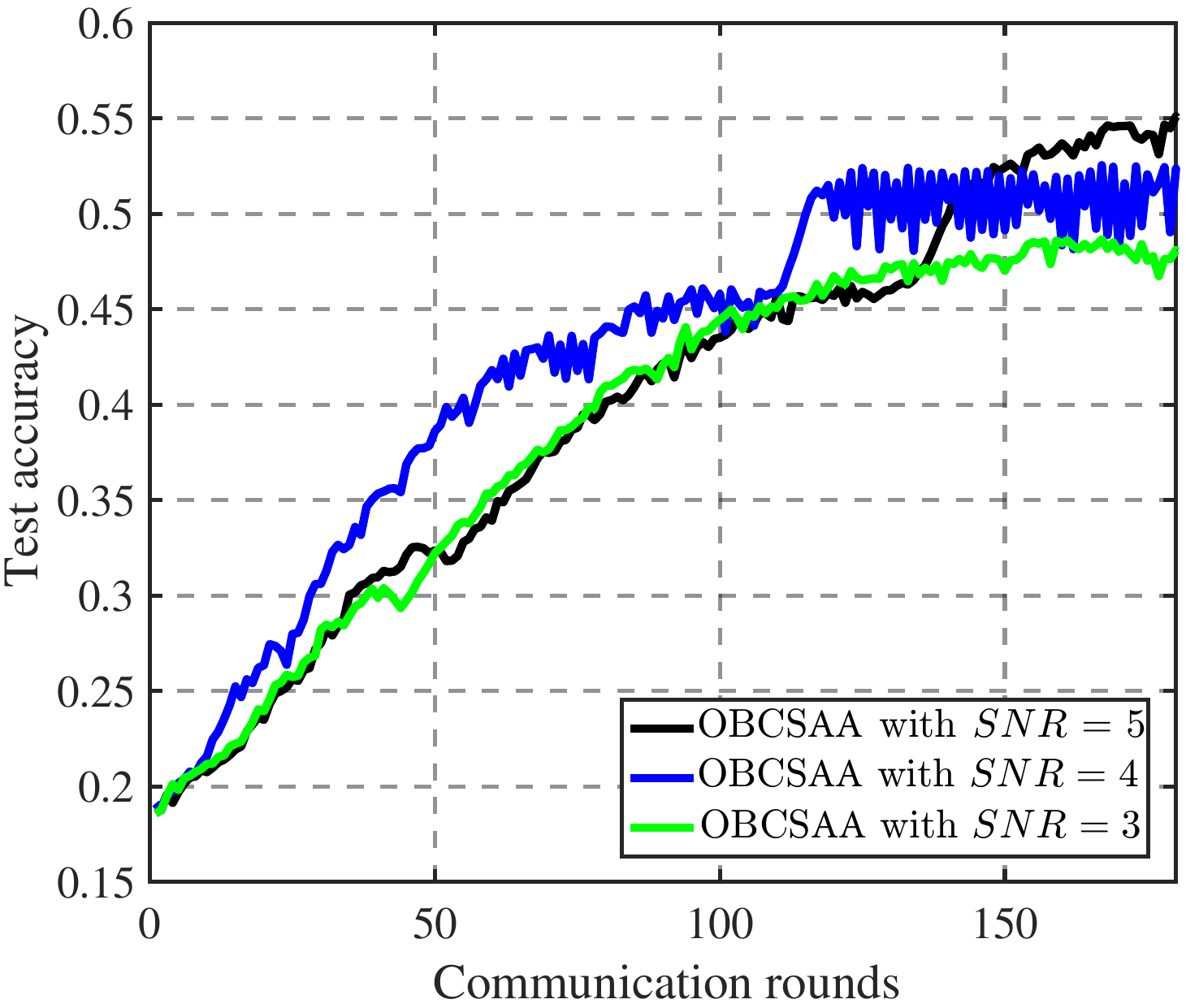}\label{fig:Accuracy_sigma}}
  \caption{The performance of our proposed OBCSAA under different the noise variance.}\label{fig:sigma}
\end{figure}

In Fig. \ref{fig:sigma}, we explore the performance of our proposed OBCSAA under different the noise variance., i.e., different $SNR$. As expected, as the noise variance increases, i.e., $SNR$ decreases, the performance of our proposed OBCSAA decreases. This is because the larger noise variance is, the more errors would be introduced in the training procedure.

\section{Conclusion}\label{Sec:Conclusion}
This paper studies a communication-efficient FL based on 1-bit CS and analog aggregation transmissions. A closed-form expression is derived for the expected convergence rate of the FL algorithm. This theoretical result reveals the tradeoff between convergence performance and communication efficiency as a result of the aggregation errors caused by sparsification, dimension reduction, quantization, signal reconstruction and noise. Guided by this revelation, a joint
optimization problem of communication and learning is developed to mitigate aggregation errors, which results in an optimal worker selection and power control. An enumeration-based method and an ADMM method are proposed to solve this challenging non-convex problem, which can obtain the optimal solution for small-scale networks and sub-optimal solution for large-scale networks, respectively. Simulation results show that our proposed FL can greatly improve communication efficiency while ensuring desired learning performance.

\section*{Acknowledgments}
We are very grateful to all reviewers who have helped improve the quality of this paper. This work was partly supported by the National Natural Science Foundation of China (Grant Nos. 61871023 and 61931001), Beijing Natural Science Foundation (Grant No. 4202054), and the National Science Foundation of the US (Grant Nos. 1741338 and 1939553).

\begin{appendices}
\section{Proof of \textbf{Lemma \ref{Lemma1}}}\label{Appendix A}
\begin{proof}\label{prlemma1}
Under the \textbf{Assumption 4}, the sparsification error $\mathbf{e}^s_{i,t} \in \mathbb{R}^D, \forall  i, t$ satisfies
\begin{align}\label{eq:e_s}
\mathbb{E}\|\mathbf{e}^s_{i,t}\|^2=\mathbb{E}\|\tilde{\mathbf{g}}_{i,t}-\mathbf{g}_{i,t}\|^2\leq (1+\delta)\frac{D-\kappa}{D}G^2, \ i=1,...,U.
\end{align}

Since $\bm{\Phi}$ satisfies the RIP condition \cite{candes2008restricted},
\begin{align}\label{eq:RIPcon}
(1-\delta)\|\mathbf{x}\|^2 \leq \|\bm{\Phi}\mathbf{x}\|^2\leq (1+\delta)\|\mathbf{x}\|^2,
\end{align}
where $\mathbf{x}$ is a $k$-sparse vector, then the quantization error $\mathbf{e}^q_{i,t} \in \mathbb{R}^S$ is derived as
\begin{align}\label{eq:e_qbound}
\mathbb{E}\|\mathbf{e}^q_{i,t}\|^2 & = \mathbb{E}\|\texttt{sign}(\bm{\Phi} \tilde{\mathbf{g}}_{i,t})-\bm{\Phi} \tilde{\mathbf{g}}_{i,t}\|^2
\\ \nonumber
&\leq \mathbb{E}(\|\texttt{sign}(\bm{\Phi} \tilde{\mathbf{g}}_{i,t})\|^2+\|\bm{\Phi} \tilde{\mathbf{g}}_{i,t}\|^2)
\\ \nonumber
&\leq
S+(1+\delta)\frac{D-\kappa}{D}G^2.
\end{align}

When the PS obtains $\mathbf{\hat{y}}^{desired}_t$ in \eqref{eq:gt}, it reconstructs the signal $\hat{\mathbf{g}}_t$, in the presence of norm-limited measurement error $\mathbf{e}^r_{t}$. It has been shown that robust reconstruction can be achieved by solving \cite{boufounos20081}:
\begin{align}\label{eq:recoverPro}
\hat{\mathbf{g}}_t=\arg \min_{\tilde{\mathbf{g}}_t} \|\tilde{\mathbf{g}}_t\|_1  \quad \texttt{s.t.}\ \|\mathbf{\hat{y}}^{desired}_t-\bm{\Phi} \tilde{\mathbf{g}}_{t}\|^2\leq \varepsilon_t
\end{align}
where $\varepsilon_t$ is the norm-limited boundary, which is given by
\begin{align} \nonumber
\mathbb{E}\|\mathbf{\hat{y}}^{desired}_t-\bm{\Phi} \tilde{\mathbf{g}}_{t}\|^2
=&\mathbb{E}\left\|\mathbf{\hat{y}}^{desired}_t-\frac{\sum_{i=1}^U K_i\beta_{i,t} (\bm{\Phi} \tilde{\mathbf{g}}_{i,t})}{\sum_{i=1}^U K_i\beta_{i,t}}\right\|^2\\ \nonumber
  =&\mathbb{E}\left\|\frac{\sum_{i=1}^U K_i\beta_{i,t} \mathbf{e}^q_{i,t}}{\sum_{i=1}^U K_i\beta_{i,t}}+\frac{\mathbf{z}_{t}}{\sum_{i=1}^U K_i\beta_{i,t}b_{t}}\right\|^2
  \\ \nonumber
  =&\mathbb{E}\left\|\mathbf{e}^q_{1,t}+\frac{\mathbf{z}_{t}}{\sum_{i=1}^U K_i\beta_{i,t}b_{t}}\right\|^2
  \\ \nonumber
  \leq& \mathbb{E}\|\mathbf{e}^q_{1,t}\|^2
  +\mathbb{E}\left\|\frac{\mathbf{z}_{t}}{\sum_{i=1}^U K_i\beta_{i,t}b_{t}}\right\|^2
  \\ \nonumber
  \leq& S+(1+\delta)\frac{D-\kappa}{D}G^2
  +\frac{S\sigma^2}{\left(\sum_{i=1}^U K_i\beta_{i,t}b_{t}\right)^{2}}\\ \label{eq:e_rbound}
  \doteq&\varepsilon_t.
\end{align}


In this case, the reconstruction error norm is bounded by
\begin{align}\label{eq:boundedreconstruction}
\|\hat{\mathbf{g}}_t-\tilde{\mathbf{g}}_t\|^2\leq \frac{C^2}{S} \varepsilon_t,
\end{align}
where $C$ is the constant depending on the properties of the measurement matrix $\bm{\Phi}$ but not on the signal \cite{candes2006stable}. According to the \textbf{Theorem 1.2} in \cite{candes2008restricted}, if $\bm{\Phi}$ has $\delta\leq \sqrt{2}-1$, $C$ can be given by
\begin{align}\label{eq:C}
C=\frac{2\varpi}{1-\varrho},
\end{align}
where $\varpi=\frac{2\sqrt{1+\delta}}{\sqrt{1-\delta}}$ and $\varrho=\frac{\sqrt{2}\delta}{1-\delta}$.

It is noted that $\tilde{\mathbf{g}}_t$ in \eqref{eq:boundedreconstruction} is the desired sparse global gradient after the worker selection. As a result, the total error at the $t$-th iteration in FL is given by
\begin{align}\label{eq:et}
\mathbb{E}\|\mathbf{e}_{t}\|^2=&\mathbb{E}(\|\hat{\mathbf{g}}_t-\mathbf{g}_t\|^2)= \mathbb{E}(\|\hat{\mathbf{g}}_t-(\tilde{\mathbf{g}}_t+\mathbf{e}^s_{t})\|^2)\leq \mathbb{E}(\|\hat{\mathbf{g}}_t-\tilde{\mathbf{g}}_t\|^2 +\|\mathbf{e}^s_{t}\|^2)
\\ \nonumber &\leq \frac{C^2}{S} \varepsilon_t + \sum_{i=1}^U\beta_{i,t}(1+\delta)\frac{D-\kappa}{D}G^2\\ \nonumber &=C^2 \left(1+(1+\delta)\frac{D-\kappa}{SD}G^2
  +\frac{\sigma^2}{\left(\sum_{i=1}^U K_i\beta_{i,t}b_{t}\right)^{2}}\right) + \sum_{i=1}^U\beta_{i,t}(1+\delta)\frac{D-\kappa}{D}G^2,
\end{align}
where $\mathbf{e}^s_{t}=\sum_{i=1}^U\beta_{i,t}\mathbf{e}^s_{i,t}$.
\end{proof}

\section{Proof of \textbf{Theorem \ref{Theorem1}}}\label{Appendix B}
\begin{proof}
To prove \textbf{Theorem \ref{Theorem1}}, we first rewrite $F(\mathbf{w}_{t})$ as the expression of its second-order Taylor expansion, which is given by
\begin{align}\label{Taylor}
F(\mathbf{w}_{t})
&=F(\mathbf{w}_{t-1})+(\mathbf{w}_{t}-\mathbf{w}_{t-1})^T\nabla F(\mathbf{w}_{t-1})+\frac{1}{2}(\mathbf{w}_{t}-\mathbf{w}_{t-1})^T\nabla^2 F(\mathbf{w}_{t-1})(\mathbf{w}_{t}-\mathbf{w}_{t-1})\nonumber\\
&\overset{(a)}{\leq} F(\mathbf{w}_{t-1})+(\mathbf{w}_{t}-\mathbf{w}_{t-1})^T\nabla F(\mathbf{w}_{t-1})+\frac{L}{2}\parallel\mathbf{w}_{t}-\mathbf{w}_{t-1}\parallel^2,
\end{align}
where \textbf{Assumption 2} is applied in the step (a).

After recovering the desired $\hat{\mathbf{g}}_{t}$ from the received signal by solving \eqref{eq:recoverPro}, then the common model is updated by
\begin{align}\label{eq:updatew}
\mathbf{w}_{t}=& \mathbf{w}_{t-1}-\alpha \hat{\mathbf{g}}_{t}
\nonumber\\
=&\mathbf{w}_{t-1}-\alpha(\nabla F(\mathbf{w}_{t-1})-\mathbf{o}),
\end{align}
where
\begin{align}\label{14}
\mathbf{o}=&\nabla F(\mathbf{w}_{t-1})-\hat{\mathbf{g}}_{t}.
\end{align}


Given the learning rate $\alpha=\frac{1}{L}$ (a special setting for simpler expression without losing the generality),
then the expected optimization function of $\mathbb{E}[F(\mathbf{w}_{t})]$ from \eqref{Taylor} can be expressed as
\begin{align}\label{eq:Egt}
\mathbb{E}[F(\mathbf{w}_{t})]\leq & \mathbb{E}\bigg[F(\mathbf{w}_{t-1})-\alpha(\nabla F(\mathbf{w}_{t-1})-\mathbf{o})^T\nabla F(\mathbf{w}_{t-1})
+\frac{L\alpha^2}{2}\parallel\nabla F(\mathbf{w}_{t-1})-\mathbf{o}\parallel^2\bigg]\nonumber\\
\overset{(b)}{=}&\mathbb{E}[F(\mathbf{w}_{t-1})]-\frac{1}{2L}\parallel\nabla F(\mathbf{w}_{t-1})\parallel^2+\frac{1}{2L}\mathbb{E}[\parallel \mathbf{o}\parallel^2],
\end{align}
where the step (b) is derived from the fact that
\begin{align}
\frac{L\alpha^2}{2}\parallel\nabla F(\mathbf{w}_{t-1})-\mathbf{o}&\parallel^2=\frac{1}{2L}\parallel\nabla F(\mathbf{w}_{t-1})\parallel^2-\frac{1}{L}\mathbf{o}^T\nabla F(\mathbf{w}_{t-1})+\frac{1}{2L}\parallel \mathbf{o}\parallel^2.
\end{align}

According to \eqref{eq:et}, $\|\mathbf{e}_{t}\|^2\leq \frac{C^2}{S} \varepsilon_{t} + \sum_{i=1}^U\beta_{i,t}(1+\delta)\frac{D-\kappa}{D}G^2$.
Then we derive $\mathbb{E}[\parallel \mathbf{o}\parallel^2]$ as follows
\begin{align}\label{eq:Eo}
\mathbb{E}[\parallel \mathbf{o}\parallel^2]=&\mathbb{E}[\|\nabla F(\mathbf{w}_{t-1})-\hat{\mathbf{g}}_{t}\|]
\nonumber\\ \nonumber
=&\mathbb{E}[\|\nabla F(\mathbf{w}_{t-1})-\mathbf{g}_{t}-\mathbf{e}_{t}\|]
\\ \nonumber
=&\mathbb{E}\Bigg[\bigg\|\frac{\sum_{i=1}^U\sum_{k=1}^{K_i}\nabla f(\mathbf{w}_{t-1};\mathbf{x}_{i,k},\mathbf{y}_{i,k})}{K}
\\ \nonumber
&-(\sum_{i=1}^U K_i\beta_{i,t})^{-1} \sum_{i=1}^U\sum_{k=1}^{K_i}\beta_{i,t}\nabla f(\mathbf{w}_{t-1};\mathbf{x}_{i,k},\mathbf{y}_{i,k}) - \mathbf{e}_{t}\bigg\|^2\Bigg]
\\
\leq& \mathbb{E}\Bigg[\bigg\|\frac{\sum_{i=1}^U\sum_{k=1}^{K_i}\nabla f(\mathbf{w}_{t-1},\mathbf{x}_{i,k},\mathbf{y}_{i,k})(1-\beta_{i,t})}{K} -\mathbf{e}_{t}\bigg\|^2\Bigg],
\end{align}

Applying the triangle inequality of norms: $\| \mathbf{X}+\mathbf{Y}\| \leq \| \mathbf{X}\|+\| \mathbf{Y}\|$, and the submultiplicative property of norms: $\| \mathbf{X}\mathbf{Y}\| \leq \| \mathbf{X}\|\| \mathbf{Y}\|$, we further derive \eqref{eq:Eo} as follows
\begin{align}\label{eq:Eo1}
\mathbb{E}[\parallel \mathbf{o}\parallel^2]
\leq &\mathbb{E}\Bigg[\frac{\sum_{i=1}^U\sum_{k=1}^{K_i}\parallel\nabla f(\mathbf{w}_{t-1};\mathbf{x}_{i,k},\mathbf{y}_{i,k})(1-\beta_{i,t})\parallel^2}{K}\Bigg]
+\mathbb{E}[\|\mathbf{e}_{t-1}\|^2] \nonumber\\ \nonumber
\leq &\mathbb{E}\Bigg[\frac{\sum_{i=1}^U\sum_{k=1}^{K_i}\parallel\nabla f(\mathbf{w}_{t-1};\mathbf{x}_{i,k},\mathbf{y}_{i,k})\parallel^2(1-\beta_{i,t})^2}{K}\Bigg] +\mathbb{E}[\|\mathbf{e}_{t-1}\|^2]\\
\leq &\frac{\sum_{i=1}^U\sum_{k=1}^{K_i}\parallel\nabla f(\mathbf{w}_{t-1};\mathbf{x}_{i,k},\mathbf{y}_{i,k})\parallel^2(1-\beta_{i,t})}{K}
+\frac{C^2}{S} \varepsilon_{t} + \sum_{i=1}^U\beta_{i,t}(1+\delta)\frac{D-\kappa}{D}G^2.
\end{align}

Applying \eqref{eq:bound} in \textbf{Assumption 3} to \eqref{eq:Eo1}, we further derive the following result as
\begin{align}\label{eq:Eo2}
\mathbb{E}[\parallel \mathbf{o}\parallel^2]
\leq&\frac{1}{K}\sum_{i=1}^U K_i(\rho_1+\rho_2\parallel\nabla F(\mathbf{w}_{t-1})\parallel^2) (1-\beta_{i,t})+\frac{C^2}{S} \varepsilon_{t} + \sum_{i=1}^U\beta_{i,t}(1+\delta)\frac{D-\kappa}{D}G^2.
\end{align}

Substituting \eqref{eq:Eo2} to \eqref{eq:Egt}, we have:
\begin{align}\label{eq:Egt1}
\mathbb{E}[F(\mathbf{w}_{t})]
\leq&\frac{1}{2L}\Bigg(\frac{1}{K}\sum_{i=1}^U K_i(\rho_1+\rho_2\parallel\nabla F(\mathbf{w}_{t-1})\parallel^2) (1-\beta_{i,t})
+\frac{C^2}{S} \varepsilon_{t} + \sum_{i=1}^U\beta_{i,t}(1+\delta)\frac{D-\kappa}{D}G^2\Bigg) \nonumber\\\nonumber&+\mathbb{E}[F(\mathbf{w}_{t-1})]-\frac{1}{2L}\parallel\nabla F(\mathbf{w}_{t-1})\parallel^2\\ \nonumber
=&\mathbb{E}[F(\mathbf{w}_{t-1})]
+\bigg(\frac{\sum_{i=1}^U K_i\rho_2(1-\beta_{i,t})}{2LK}-\frac{1}{2L}\bigg)\parallel\nabla F(\mathbf{w}_{t-1})\parallel^2 \\
&+\frac{\sum_{i=1}^U K_i\rho_1(1-\beta_{i,t})}{2LK}+\frac{1}{2L}(\frac{C^2}{S} \varepsilon_{t} + \sum_{i=1}^U\beta_{i,t}(1+\delta)\frac{D-\kappa}{D}G^2).
\end{align}

Summing up the inequality above from $t = 1$ to $t = T$, we get
\begin{align}\label{eq:EgtSumming}
\mathbb{E}[F(\mathbf{w}_{t})&-F(\mathbf{w}_0)]
\leq -\sum_{t=1}^{T}A_t\parallel\nabla F(\mathbf{w}_{t-1})\parallel^2 +\sum_{t=1}^{T}B_t,
\end{align}
where
\begin{align}\label{eq:At}
A_{t}=\frac{1}{2L}-\frac{\sum_{i=1}^U K_i\rho_2(1-\beta_{i,t})}{2LK},
\end{align}
\begin{align}\label{eq:Bt}
B_{t}&=\frac{\sum_{i=1}^U K_i\rho_1(1-\beta_{i,t})}{2LK}+\frac{1}{2L}(\frac{C^2}{S} \varepsilon_{t} + \sum_{i=1}^U\beta_{i,t}(1+\delta)\frac{D-\kappa}{D}G^2).
\end{align}
The inequality \eqref{eq:EgtSumming} can be also written as
\begin{align}
\sum_{t=1}^{T}A_t\parallel\nabla F(\mathbf{w}_{t-1})\parallel^2
\leq& \mathbb{E}[F(\mathbf{w}_0)-F(\mathbf{w}_{t})] +\sum_{t=1}^{T}B_t\label{eq:EgtSumming1}
\leq \mathbb{E}[F(\mathbf{w}_0)-F(\mathbf{w}^*)] +\sum_{t=1}^{T}B_t.
\end{align}

Since $\frac{1-\rho_2}{2L}\leq A_t\leq \frac{1}{2L}$, we have
\begin{align}
\frac{1}{T}\sum_{t=1}^{T}\frac{1-\rho_2}{2L}\parallel\nabla &F(\mathbf{w}_{t-1})\parallel^2
\leq \frac{1}{T}\sum_{t=1}^{T}A_t\parallel\nabla F(\mathbf{w}_{t-1})\parallel^2 \label{eq:EgtSumming2}
\leq \frac{1}{T}\mathbb{E}[F(\mathbf{w}_0)-F(\mathbf{w}^*)] +\frac{1}{T}\sum_{t=1}^{T}B_t.
\end{align}

As a result, we get
\begin{align}
\frac{1}{T}\sum_{t=1}^{T}\frac{}{}\parallel\nabla F(\mathbf{w}_{t-1})\parallel^2
\leq& \frac{2L}{T(1-\rho_2)}\mathbb{E}[F(\mathbf{w}_0)-F(\mathbf{w}^*)] \label{eq:EgtSumming3}+\frac{2L}{T(1-\rho_2)}\sum_{t=1}^{T}B_t.
\end{align}
The proof is completed.
\end{proof}

\end{appendices}

\bibliographystyle{IEEEtran}
\bibliography{ref}
\end{document}